\documentclass[lettersize,journal]{IEEEtran}
\usepackage{amsmath,amsfonts}
\usepackage{algorithmic}
\usepackage{algorithm}
\usepackage{array}
\usepackage{textcomp}
\usepackage{stfloats}
\usepackage{url}
\usepackage{verbatim}
\usepackage{graphicx}
\usepackage{cite}
\usepackage{color}

\usepackage{multirow}
\usepackage{amssymb}
\usepackage{times}
\usepackage{soul}
\usepackage{url}
\usepackage{dsfont}
\usepackage{amsthm}
\usepackage{booktabs}
\usepackage{subcaption}
\usepackage{makecell}
\newtheorem{remark}{Remark}
\theoremstyle{definition}
\newtheorem{thm}{Theorem}
\newtheorem{prop}[thm]{Proposition}
\newtheorem{dfn}[thm]{Definition}
\begin{document}

\title{Lyapunov Stable Graph Neural Flow}

\author{Haoyu Chu, Xiaotong Chen, Wei Zhou,~\IEEEmembership{Senior Member,~IEEE}, Wenjun Cui, Kai Zhao, Shikui Wei, Qiyu Kang
\thanks{This work was supported in part by the Young Scientists Fund of the National Natural Science Foundation of China (No. 62506363) and the Fundamental Research Funds
for the Central Universities of Ministry of Education of China (No.
2025QN1151). (Corresponding author: Qiyu Kang.)

Haoyu Chu is with the School of Computer Science and Technology / School of Artificial Intelligence, China University of Mining and Technology, Xuzhou 221008, China (email: 6784@cumt.edu.cn). 

Qiyu Kang is with the School of Information Science and Technology, University of Science and Technology of China, Hefei 230026, China. (email: kang0080@e.ntu.edu.sg)}
}

\markboth{}%
{Shell \MakeLowercase{\textit{et al.}}: A Sample Article Using IEEEtran.cls for IEEE Journals}

\maketitle

\begin{abstract}
Graph Neural Networks (GNNs) are highly vulnerable to adversarial perturbations in both topology and features, making the learning of robust representations a critical challenge. 
In this work, we bridge GNNs with control theory to introduce a novel defense framework grounded in integer- and fractional-order Lyapunov stability. Unlike conventional strategies that rely on resource-heavy adversarial training or data purification, our approach fundamentally constrains the underlying feature-update dynamics of the GNN. We propose an adaptive, learnable Lyapunov function paired with a novel projection mechanism that maps the network's state into a stable space, thereby offering theoretically provable stability guarantees.
Notably, this mechanism is orthogonal to existing defenses, allowing for seamless integration with techniques like adversarial training to achieve cumulative robustness. Extensive experiments demonstrate that our Lyapunov-stable graph neural flows substantially outperform base neural flows and state-of-the-art baselines across standard benchmarks and various adversarial attack scenarios.
\end{abstract}


\section{Introduction}
Graph neural networks (GNNs) have achieved remarkable success and have found widespread application in numerous domains such as social network analysis~\cite{li2023survey,zhao2025social}, traffic flow prediction~\cite{wang2025robust, fan2025emerging}, graph drawing~\cite{tiezzi2024graph}, and molecular chemistry~\cite{reiser2022graph}. Conventional discrete GNNs, such as graph convolutional networks (GCNs)~\cite{thomas2017semi} and graph attention networks (GATs)~\cite{velivckovic2017graph}, propagate information between nodes through successive network layers. In recent years, driven by advances in continuous-time modeling techniques, GNN design guided by dynamical systems theory has emerged as a significant trend~\cite{zheng2025survey,liu2025graph}. The core idea underlying these kinds of GNN is to model the dynamic evolution of the graph feature via ordinary or partial differential equations. Notable methods in this line of work include GRAND~\cite{chamberlain2021grand}, which formulates information propagation via a heat diffusion equation; GRAND++~\cite{thorpe2022grand++}, extending GRAND with an additional source term; GREAD~\cite{choi2023gread}, which introduces a reaction term into the diffusion equation; and GraphCON~\cite{rusch2022graph}, leveraging a second-order ordinary differential equation for dynamics modeling. These physically grounded formulations enhance the interpretability of the model and offer intuitive guidance for optimization and hyperparameter selection. Recognizing the demonstrated advantages of fractional differential operators in capturing complex real-world phenomena~\cite{pu2017fractional, chakraverty2022computational}, Kang et al.~\cite{kang2024unleashing} proposed FROND, a framework leveraging the non-local properties of fractional calculus to capture long-range dependencies during feature updates. Building on this, Cui et al.~\cite{cui2025neural} introduced a flexible mechanism by parameterizing variable-order fractional derivatives of hidden states through an auxiliary network.

Despite these promising achievements, the security of GNNs faces severe threats from adversarial attacks. Adversarial examples, first introduced in computer vision~\cite{szegedy2014intriguing}, involve imperceptible perturbations to inputs that cause misclassifications. Researchers have demonstrated that GNNs are similarly susceptible to adversarial attacks on the graph~\cite{GUAN2024112642, li2026generative}. While several graph adversarial defense strategies have been proposed to mitigate this issue, the underlying mechanisms remain unclear. In contrast, the dynamical systems perspective has offered valuable insights into the inherent stability properties of graph-based learning, leading to the development of new approaches aimed at enhancing GNN robustness. Yang et al.~\cite{song2022robustness} pioneered the robustness analysis of graph neural diffusion, theoretically and empirically demonstrating that such models exhibit intrinsic stability against graph perturbations. Subsequently, Kang et al.~\cite{kang2024coupling} demonstrated FROND's robustness from the perspective of fractional numerical stability. Zhao et al.~\cite{zhao2023adversarial} proposed Hamiltonian-inspired graph neural flows (HANG), leveraging energy conservation principles to enhance stability and mitigate the oversmoothing problem. 

In control theory, Lyapunov stability theory establishes a foundational framework for evaluating system responses to initial perturbations, with established extensions to fractional-order dynamical systems. Furthermore, Lyapunov theory has found applications in analyzing the stability of deep learning models~\cite{storm2024finite, chu2024structure}. This naturally raises the question: Can Lyapunov stability theory be leveraged to provide rigorous theoretical guarantees for the stability of dynamical system-based GNNs? Specifically, Lyapunov methods comprise two primary approaches: the indirect method and the direct method~\cite{slotine1991applied}. The indirect method involves linearizing the nonlinear dynamical system near its equilibrium state and then judging stability based on the eigenvalues of the linearized system. The direct method characterizes stability by the Lyapunov function (an abstract energy function), where a Lyapunov stable dynamical system implies that its Lyapunov function exists and it is positive definite and decreasing over time.

To this end, unlike conventional defense methods such as adversarial training and detection, this work considers the robustness of GNNs from the dynamical systems perspective, providing a brand new certified robustness enhancement approach based on the Lyapunov stability theory. We aim to constrain the graph neural flow to satisfy the conditions of Lyapunov's direct method, which is simpler to implement than the indirect method and can provide stricter stability criteria locally~\cite{cui2023lyapunov}. In addition, a special classification layer that can separate the equilibrium points is proposed to maximize the separation between the equilibrium points of different classes. As a result, when facing an input perturbation, the integer- and fractional-order Lyapunov stable graph flow have the resilience to evolve toward stable equilibrium points. Experimental results demonstrate that the proposed models significantly outperform both the baseline integer and fractional graph neural flow models. Furthermore, since our method is orthogonal to existing adversarial defenses, it can be combined with adversarial training to further enhance the robustness of GNN models. 

The remainder of this paper is structured as follows. Section II introduces related works on adversarial attacks and defenses on graphs, as well as the application of Lyapunov theory in deep learning. Section III first establishes the necessary notations and preliminaries for graph neural networks, then provides the definition of Lyapunov stability and introduces Lyapunov's direct method, covering both the integer and fractional formulations. Section IV elaborates the proposed methodology, including the base fractional dynamical system, the integer- and fractional-order Lyapunov stability module, and the equilibrium-separating classification layer. Section V provides comprehensive experiments and analyses to verify the robustness of the proposed models across diverse adversarial scenarios. Finally, Section VI concludes the paper and outlines potential future research directions. The main contributions of this work are summarized as follows:
\begin{itemize}
\item To our knowledge, this work presents the first integration of Lyapunov’s direct method with graph neural flow, establishing a framework for stability-guaranteed learning on graph datasets.
\item The proposed methodology is applicable to both integer-order and fractional-order dynamical systems, providing a comprehensive provable robustness enhancement framework for a wide range of graph neural flow models.
\item Extensive experiments demonstrate that the proposed approach delivers consistent robustness performance gains across a wide spectrum of adversarial settings, including transductive and inductive black-box attacks as well as white-box attacks.
\item The proposed defense strategy is orthogonal to existing adversarial defense methods, enabling integration with strategies such as adversarial training to achieve further improved robustness.
\end{itemize}

\section{Related Works}
\subsection{Adversarial Attacks and Defenses on Graph}

Adversarial attacks on graph are mainly categorized according to knowledge, objectives, and approaches. Regarding knowledge, \textit{black-box} attacks lack access to the target model's internal details but can utilize graph data, while \textit{white-box} attacks posses full model information~\cite{TAO2024308}. Regarding objectives, \textit{poisoning} attacks corrupt graph data to degrade model performance during training, whereas \textit{evasion} attacks affect the inference of the trained model by altering input data~\cite{Zheng2021GraphRB}. Regarding approaches, \textit{modification} attacks alter the original graph through edge changes or feature perturbations, while \textit{injection} attacks introduce malicious nodes without manipulating the existing graph structure~\cite{chen2022hao}.

Adversarial defenses on graph include reactive methods and proactive methods. Reactive methods are characterized by their plug-and-play capability to resist adversarial attacks without retraining GNNs. For example, \textit{adversarial detection}~\cite{LIU2025103912} uses a detector to identify anomalous edges or features in the input, \textit{adversarial purification}~\cite{lee2025self} denoises adversarial perturbations in the input to restore clean graph structures or features typically by generative models. Proactive methods aim to enhance models' robustness by modifying training data or internal mechanisms. For example, \textit{graph adversarial training} methods~\cite{gosch2023adversarial} incorporate adversarial samples into the training dataset, \textit{network robustification} methods introduce stability constraints, such as Lipschitz constraints~\cite{juvina2024training} or energy conservation constraints~\cite{zhao2023adversarial}, into graph neural networks.

The proposed method falls under the category of network robustification. As an energy-based approach, we introduce energy dissipation other than conservation to drive the state of graph neural flows toward a stable equilibrium point with minimal energy.

\subsection{Application of Lyapunov theory in Deep Learning}

Recent advances have demonstrated the effectiveness of combining Lyapunov methods with neural network architectures to construct provable certificates. Manek and Kolter~\cite{kolter2019learning} introduced a scheme for the joint learning of a Lyapunov function and stable dynamics, with the application of this technique in physical simulation and video texture generation. Storm et al.~\cite{storm2024finite} developed a robustness assessment framework for multilayer perceptron using finite-time Lyapunov exponents (FTLE), which can characterize the growth or decay of local perturbations in forward propagation. Chu et al.~\cite{chu2024lyapunov} proposed a Lyapunov stable deep equilibrium model for defending adversarial attacks on images.

Following the emergence of neural ordinary differential equations (Neural ODEs)~\cite{chen2018neural}, there has been growing interest in integrating Lyapunov theory with these continuous-depth models. Kang et al.~\cite{kang2021stable} employed Lyapunov's indirect method to enhance the robustness of Neural ODEs. Schlaginhaufen et al.~\cite{schlaginhaufen2021learning} introduced a novel regularization term based on Lyapunov-Razumikhin functions to stabilize neural delay differential equations. Building on LaSalle's invariance principle (a generalization of Lyapunov's method), Takeishi et al.~\cite{takeishi2021learning} developed Neural ODEs that are capable of handling stability for general invariant sets, including limit cycles and line attractors. 

\section{Notation and Preliminaries}
We define an undirected graph without self-loops as $G:=(\mathcal{V},\mathbf{W})$, where $\mathcal{V}={1,\ldots, N}$ denotes a collection of $N$ nodes, $|\mathcal{V}|=N$, The weighted connections between nodes are encoded in an adjacency matrix $\mathbf{W}:=(W_{ij})\in\mathbb{R}^{N\times N}$, where each entry $W_{ij}$ corresponds to the weight of the edge linking node $i$ to node $j$, satisfying $W_{ij}=W_{ji}$ (that is, $\mathbf{W}$ is symmetric). Let $\mathbf{U}=\bigl([\boldsymbol{u}^{(1)}]^\top, \cdots, [\boldsymbol{u}^{(N)}]^\top \bigr)^\top\in\mathbb{R}^{|\mathcal{V}|\times d}$ be the features of the node, which is a matrix composed of row vectors $(\boldsymbol{u}^{(i)})^\top \in \mathbb{R}^d$, with $i$ being the index of the node. 

A typical integer-order dynamical system can be described by a first-order ODE:
\begin{equation}
    \frac{\mathrm d\boldsymbol{u}}{\mathrm dt} = f(t,\boldsymbol{u}(t)), \label{eq.1ode}
    \end{equation}
where $\boldsymbol{u}(t)$ represents the state of the system at time $t$, and $f$ defines the dynamics that governs the evolution of the system. If the function $f$ does not explicitly depend on time $t$, the dynamical system is autonomous. In the analysis of such dynamical systems, the Lyapunov stability theory provides a fundamental framework for studying the behavior of solutions without explicitly solving the differential equations.

\begin{dfn}[Lyapunov stability] 
    \label{def1}
    An equilibrium $\boldsymbol{u}^{\star}$ is called stable in the Lyapunov sense if, for any given $\varepsilon>0$, there exists $\delta>0$ such that whenever the initial condition satisfies $\|\boldsymbol{u}(0)-\boldsymbol{u}^{\star}\|<\delta$, the solution remains within $\|\boldsymbol{u}(t)-\boldsymbol{u}^{\star}\|<\varepsilon$ for all $t \geq 0$. Suppose $\boldsymbol{u^{\star}}$ is stable and satisfy $\lim_{t \rightarrow \infty}\Vert \boldsymbol{u}(t) - \boldsymbol{u^{\star}}\Vert = 0$, it is termed \textit{asymptotically} stable. Further, suppose $\boldsymbol{u^{\star}}$ is stable and there exists some $\nu>0$ such that $\lim_{t \rightarrow \infty}\Vert \boldsymbol{u}(t) - \boldsymbol{u^{\star}}\Vert e^{\nu t}= 0$, it is termed \textit{asymptotically} stable.
\end{dfn}

\begin{thm}[Lyapunov's direct method~\cite{giesl2015review}]
\label{thm:int}
Given a fixed point $\boldsymbol{u}^{\star}$, suppose there exists a continuously differentiable function $V: [0, \infty) \times \mathcal{U} \rightarrow \mathds{R}$ defined for $t \ge 0$ in a neighborhood $\mathcal{U}$ of $\boldsymbol{u}^{\star}$, fulfilling the following conditions:
\begin{enumerate}
\item[(1)]
$V$ attains its minimum at $\boldsymbol{u}^{\star}$ for all $t \ge 0$.
\vspace{0.5em}
\item[(2)]
Along every solution trajectory in $\mathcal{U}$ other than the equilibrium $\boldsymbol{u}^{\star}$, $V$ decreases strictly.
\end{enumerate}
then $V$ is called a Lyapunov function, and the equilibrium $\boldsymbol{u}^{\star}$ is \textit{asymptotically} stable. Furthermore, $\boldsymbol{u}^{\star}$ is \textit{exponentially} stable if, in addition, there exist positive constants $c$ and $K(\alpha)$ such that:
    \begin{enumerate}
    \item[(3)]  
    $ \| \boldsymbol{u} \|^2 \leq V(t,\boldsymbol{u}) \leq K(\alpha) \| \boldsymbol{u} \|^2 $, $\forall \boldsymbol{u} \in \{ \boldsymbol{u}:\| \boldsymbol{u} \| \leq \alpha \}$;
    \item[(4)] 
    \vspace{0.5em}
    $\dot{V}({t,\boldsymbol{u}})  \leq -cV(t, \boldsymbol{u})$, $c>0$.
    \end{enumerate}
\end{thm}

To accommodate a broader range of application scenarios, we also consider fractional-order systems. Let $\boldsymbol{u}$ be a state vector in a state space, the dynamics of the non-autonomous fractional-order system is described by 
\begin{equation}
    D_t^\beta \boldsymbol{u}(t)=f(t,\boldsymbol{u}). \label{eq.fode}
\end{equation}
Here, we adopt the following definition of the Caputo fractional derivative~\cite{diethelm2010analysis}: 
\begin{equation} 
    D_t^\beta \boldsymbol{u}(t):=\frac{1}{\mathrm{\Gamma}(1-\beta)}\int_{0}^{t}{{(t-\tau)}^{-\beta}\boldsymbol{u}^\prime(\tau)d\tau},
\end{equation}
where $\beta\in(0,1]$ denotes fractional order, $\mathrm{\Gamma}$ is the Gamma function, , and $\boldsymbol{u}^\prime(\cdot)$ denotes the first derivative of a differentiable function $\boldsymbol{u}(t)$. When $\beta = 1$, $D_t^1 \boldsymbol{u}(t) = {d \boldsymbol{u}}/{dt}$, and Eq.~(\ref{eq.fode}) reduces to the first-order ODE (\ref{eq.1ode}). To demonstrate the fractional-order extension of Lyapunov's direct method, we first introduce the concept of Mittag–Leffler stability. 

\begin{dfn}[Mittag–Leffler stability]
A fractional dynamical system is called Mittag–Leffler stable if
\begin{equation}
    \| \boldsymbol{u}(t)\| \leq {m[\boldsymbol{u}(0)]E_\beta(-\lambda(t)^\beta)}^b,
\end{equation}
where $\lambda\geq0$, $b>0$, the function $m(\cdot)$ fulfills $m(0)=0$, $m(\boldsymbol{u})\geq0$, and is locally Lipschitz on $\boldsymbol{u}\in\mathbb{R}^{n}$. Here
\begin{equation}
    E_\beta(z) = \sum_{k=0}^{\infty}\frac{z^k}{\Gamma(k \beta + 1)}
\end{equation}
is the Mittag–Leffler function~\cite{diethelm2010analysis}. In the case where $\beta=1$, this function reduces to the familiar exponential function.
\end{dfn}

\begin{remark}
Within the framework of fractional‑order dynamics, Mittag–Leffler stability serves as a key criterion. If a system is Mittag-Leffler stable, it is inherently asymptotically stable.
\end{remark}

\begin{thm}[Fractional Lyapunov's direct method~\cite{LI20101810}] 
\label{thm:frac}
Let $\boldsymbol{u}^{\star}$ be an equilibrium of a fractional‑order system and $\mathds{D}\subset\mathbb{R}^{n}$ a region containing the origin. Assume there exists a function $V(t,\boldsymbol{u}(t)):[0,\infty)\times\mathds{D}\to\mathds{R}$ that is continuously differentiable and locally Lipschitz in $\boldsymbol{u}$, and suppose for some positive constants $\alpha_{i}(i=1,2,3),a,b$, this candidate Lyapunov function satisfies
\begin{equation} \label{eqn:lya01}
    \alpha_{1}\|\boldsymbol{u}\|^a\leq V(t,\boldsymbol{u}(t))\leq  \alpha_{2}\|\boldsymbol{u}\|^{ab},
\end{equation}
and
\begin{equation} \label{eqn:lya02}
D_t^\beta V(t,\boldsymbol{u}(t)) \leq - \alpha_{3}\|\boldsymbol{u}\|^{ab}.
\end{equation}
Then the system is Mittag–Leffler stable.
\end{thm}

\begin{remark}
The integer-order Lyapunov stability theorem emerges as a particular instance of the above result when the fractional order is set to $\beta=1$.
\end{remark}

\begin{figure*}[t]
\centering
\includegraphics[width=0.96\linewidth]{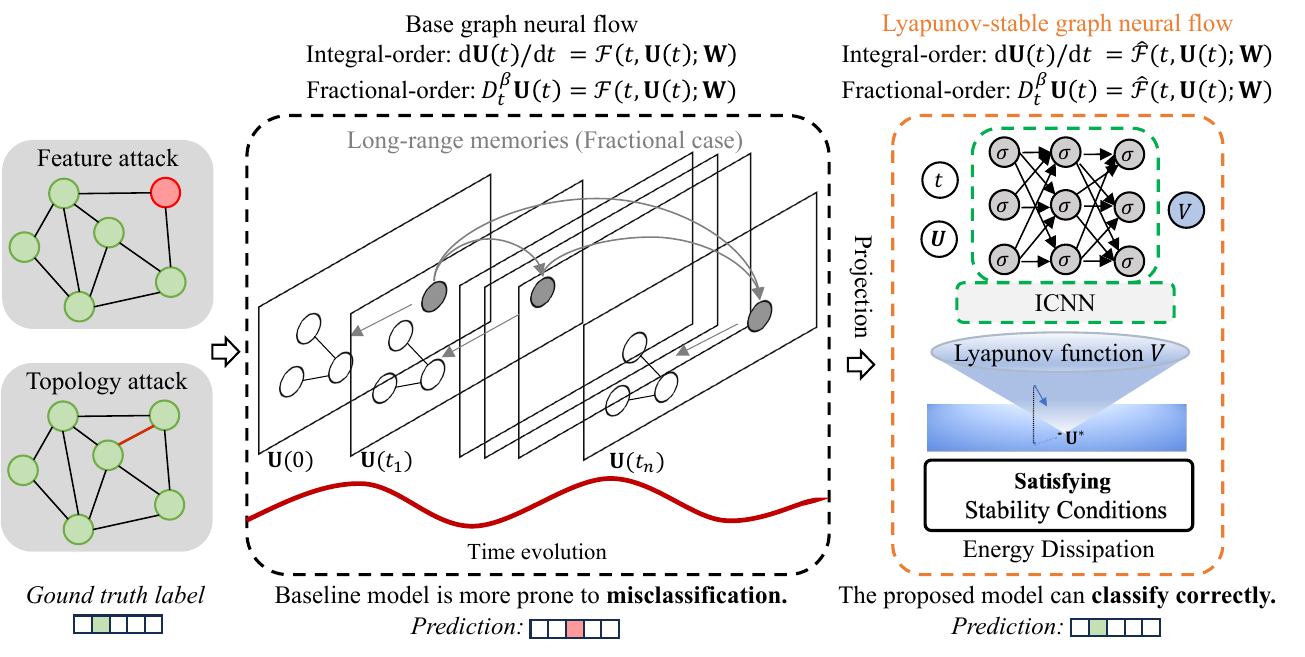}
\caption{The scratch of the Lyapunov stability module. The black dashed box illustrates the dynamic evolution of the base graph neural flow over the time interval $[0, t_n]$, where the fractional case captures long-range memories by incorporating historical states, unlike the integer-order's behavior depends solely on the current state.  The orange dashed box depicts the projected Lyapunov-stable system, which satisfies the conditions of the Lyapunov's direct method, where the Lyapunov function is implemented by an input-convex neural network, a fully connected neural network with special properties.}
\label{fig1}
\end{figure*}

\section{Methodology}

This section presents the proposed framework for enhancing the robustness of GNNs through the lens of Lyapunov stability theory. We begin by introducing the base graph neural flow, which formulates the node feature evolution as either an integer-order or fractional-order dynamical system. Next, we detail the integer- and fractional-order Lyapunov stability module, a novel mechanism that projects the base dynamics into a Lyapunov-stable space via a learnable Lyapunov function. We then describe the architecture and training pipeline of the resulting Lyapunov stable GNNs, which integrate the stability module with an equilibrium-separating classification layer to maximize inter-class margin while preserving stability. Finally, we provide a theoretical analysis of certified robustness under perturbation, demonstrating that the proposed models exhibit robustness, and thus provable resilience against adversarial perturbations.

\subsection{The Base Graph Neural Flow}

The node feature update process can be modeled as a time-evolving feature extraction procedure, where the differential equation describes the information propagation between nodes. Let $\mathbf{U}(t)$ denote the node features at time $t$, and let $\mathcal{F}(\cdot)$ represent the governing equation, the dynamic evolution of graph nodes can be expressed as a system of ODEs:

\begin{equation}
    \frac{\text{d}\mathbf{U}(t)}{\text{d} t}=\mathcal{F}(t,\mathbf{U}(t);\mathbf{W})
\end{equation}
with the initial condition $\mathbf{U}(0)$, where 
\begin{equation*}
    \mathbf{U}(t)=\left(\left[\boldsymbol{u}^{(1)}(t)\right]^\top,\cdots,\left[\boldsymbol{u}^{(N)}(t)\right]^\top\right)^\top.
\end{equation*}

Integer-order differential equations govern systems whose behavior depends solely on the current state, whereas fractional-order differential equations (FDEs) correlate a series of past states with the current state. This property grants them a distinct advantage in modeling non-uniform dynamical behaviors in complex networks, enabling more accurate characterization of cumulative memory effects and non-local properties. By incorporating the Caputo fractional derivative, this dynamic process can be reformulated as a fractional-order ODE: 
\begin{equation}
    D_t^\beta\mathbf{U}(t)=\mathcal{F}(t,\mathbf{U}(t);\mathbf{W})
\end{equation}
with initial condition $\mathbf{U}(0)$. The final node features $\mathbf{U}(T)$, which can be used for node classification, can be obtained by solving this equation up to a terminal time $t=T$.

\subsection{Integer- and Fractional-Order Lyapunov Stability Module}

Due to inherent cognitive deficiencies of deep neural networks, the robustness of both integer-order and fractional-order differential equation-based GNNs remains challenging to guarantee. We consider the failure of the current GNNs on adversarial examples to be the neglect of preserving the underlying stability structure. Therefore, we propose to impose a hard constraint concerning Lyapunov stability on the graph neural flow and introduce a family of Lyapunov-stable GNNs, which encompass the integer-order Lyapunov-stable GNNs (IL-GNNs) and the fractional-order Lyapunov-stable GNNs (FL-GNNs).

The core innovation lies in the unified Lyapunov stability module, which integrates graph neural flow with Lyapunov theory to ensure the robustness of the GNN dynamics. This module applies Lyapunov's direct method to ensure that the node features evolve in a stable manner, avoiding unbounded growth or oscillations, thereby enhancing resilience against adversarial attacks. Specifically, the stability module is designed by constructing a Lyapunov function $V$ that satisfies the conditions of Theorem~\ref{thm:int} for the case of integer-order (\textit{int}) or Theorem~\ref{thm:frac} for the case of fractional-order (\textit{frac}) and projecting the base GNN dynamics onto a stable manifold. Figure~\ref{fig1} shows the scratch of the fractional Lyapunov stability module.

Let $\mathbf{U}^{\star}$ be the equilibrium points of the GNN dynamics, the Lyapunov-stable module is defined as
\begin{equation}
\begin{aligned}
&\hat{\mathcal{F}}(t,\mathbf{U}^{\star};\mathbf{W}) \\
&= \begin{cases}\mathcal{F}(t,\mathbf{U}^{\star}) & \text { if } \phi(t,\mathbf{U}^{\star}) \leq 0, \\ \mathcal{F}(t,\mathbf{U}^{\star})-\nabla V \frac{\phi(t,\mathbf{U}^{\star})}{\|\nabla V(\mathbf{U}^{\star})\|^2} & \text {otherwise (for \textit{int})},
\\ 
\mathcal{F}(t,\mathbf{U}^{\star})-D_t^\beta V\frac{\phi(t,\mathbf{U}^{\star})}{\|D_t^\beta V\|^2} & \text {otherwise (for \textit{frac})},\end{cases} 
\end{aligned}
\label{eqn: model}
\end{equation}
where $\phi(t,\mathbf{U}^{\star})= \nabla V^\top F +c V$ for the case of integer-order and $\phi(t,\mathbf{U}^{\star}) = D_t^\beta V +\alpha_{3}\|\mathbf{U^{\star}}\|$ for the case of fractional-order. Assuming $\phi(t,\mathbf{U}^{\star}) > 0$ holds, the output of the basic model will be refined by projection operation, otherwise the output is returned unchanged. It is worth mentioning that the gradient is multiplied by the numerator, which prevents the gradient from vanishing theoretically.

The function $V$ is constructed as a trainable neural network that is designed to meet the conditions of a Lyapunov function:
\begin{equation}
V=\sigma(\text{ICNN}(\mathbf{U}^{\star})-\text{ICNN}(\mathbf{0}))+\|\mathbf{U}^{\star}\|^2,
\end{equation}
where $\sigma$ represents a positive convex activation function satisfying $\sigma(0) = 0$, and $\text{ICNN}$ refers to an input-convex neural network~\cite{amos2017input}:
\begin{equation} \label{eqn:icnn}
\begin{aligned}
\boldsymbol{q}_1 &=\sigma_0\left(W^{I}_0 \mathbf{U}^{\star}+b_0\right), \\
\boldsymbol{q}_{i+1} &=\sigma_i\left(P_i \boldsymbol{q}_i+W^{I}_i \mathbf{U}^{\star}+b_i\right), i=1, \ldots, k-1, \\
g(\mathbf{U}^{\star}) & \equiv \boldsymbol{q}_k,
\end{aligned}
\end{equation}
where $W^{I}_i$, $b_i$ denote the weights and biases of the $i$-th layer, respectively. $P_i$ are positive weights. $\sigma$ is convex and non-decreasing activation function to guarantee that each layer's operation maintains convexity, which can be chosen as smooth version of ReLU function~\cite{kolter2019learning}:
\begin{equation} \label{eqn:sigma}
\sigma(x)= \begin{cases}0, & \text { if } x \leq 0, \\ x^2 / (2 d), &  0<x<d, \\ x-d / 2, & x\geq d. \end{cases}
\end{equation}

By incorporating this projection mechanism, the resulting graph nerual flow is guaranteed to meet the requirements of the integer- or fractional-order Lyapunov's direct method. Consequently, the equilibrium points of the projected model are ensured to be Lyapunov stable for the integer-order case or Mittag–Leffler stable for the fractional-order case.

\subsection{Lyapunov Stable GNNs and their Training Process}

The proposed Lyapunov stable GNNs contain three key components, namely the base graph neural flow, the Lyapunov stability module, and an equilibrium-separating classification layer. The dynamical system parameterized by IL-GNNs and FL-GNNs can be formalized as follows: 
\begin{equation} \label{dym:int}
    \frac{\text{d}\mathbf{U}(t)}{\text{d} t}  = \operatorname{Proj}(\mathcal{F},\{\mathcal{F}: \nabla V^T F \leq - c V(t,\mathbf{U})\}),
\end{equation}
\begin{equation} \label{dym:frac}
    D_t^{\beta} \mathbf{U}(t)  = \operatorname{Proj}(\mathcal{F},\{\mathcal{F}: D_t^\beta V \leq - \alpha_{3}\|\mathbf{U}\|\}),
\end{equation}
where $\operatorname{Proj}$ denotes the projection shown in Eqn.(\ref{eqn: model}), and $\mathcal{F}$ can be chosen as one of the base dynamical systems.

\subsubsection{Choices of the base dynamical system} We consider three kinds of integer-order base dynamical system, namely GRAND~\cite{chamberlain2021grand}, GraphBel~\cite{song2022robustness} and GraphCON (GCON)~\cite{rusch2022graph}, and their corresponding fractional-order extensions as the governing equations.

GRAND and Fractional-GRAND models (F-GRAND) are governed by 
\begin{equation}
    \frac{\text{d}\mathbf{U}(t)}{\text{d} t}  = (\mathbf{A}(\mathbf{U}(t))-\mathbf{I})\mathbf{U}(t), 
\end{equation}
\begin{equation}
    D_t^{\beta} \mathbf{U}(t)=(\mathbf{A}(\mathbf{U}(t))-\mathbf{I})\mathbf{U}(t), 
\end{equation}
respectively, where $\mathbf{I}\in\mathbb{R}^{N\times N}$ is an identity matrix, $\mathbf{A}(\mathbf{U}(t))$ is a time-dependent learnable attention matrix:
\begin{equation}
	\boldsymbol{A} = a(\boldsymbol{u}_{i},\boldsymbol{u}_{j})=\mathrm{softmax}\left(\frac{(\mathbf{W}^{K} \boldsymbol{u}_{i})^\top \mathbf{W}^{Q}\boldsymbol{u}_{j}}{d_{k}}\right),
\end{equation}
where $\mathbf{W}^{K}$ and $\mathbf{W}^{Q}$ are learnable matrices, and $d_{k}$ is a hyperparameter.

GraphBel and Fractional-GraphBel models (F-GBel) are defined as
\begin{equation}
    \frac{\text{d}\mathbf{U}(t)}{\text{d} t}= (\mathbf{A}_{\mathbf{s}}(\mathbf{U}(t)) \odot \mathbf{B}_{\mathbf{s}}(\mathbf{U}(t)) - \mathbf{\Psi} (\mathbf{U}(t)))\mathbf{U}(t),
\end{equation}
\begin{equation}
    D_t^{\beta} \mathbf{U}(t)= (\mathbf{A}_{\mathbf{s}}(\mathbf{U}(t)) \odot \mathbf{B}_{\mathbf{s}}(\mathbf{U}(t)) - \mathbf{\Psi} (\mathbf{U}(t)))\mathbf{U}(t),
\end{equation}
respectively, where $\odot$ denotes the element-wise multiplication, $\mathbf{A}_{\mathbf{s}}(\cdot)$ is a learnable attention function, $\mathbf{B}_{\mathbf{s}}(\cdot)$ represents a normalization function, $\mathbf{\Psi} (\boldsymbol{u}, \boldsymbol{u})=\sum_{\boldsymbol{v}}(\mathbf{A}_{\mathbf{s}}\odot\mathbf{B}_{\mathbf{s}})(\boldsymbol{u}, \boldsymbol{v})$. 

GCON and Fractional-GraphCON model (F-GCON) can be formalized as 
\begin{equation}
	\begin{cases}
		\dfrac{\mathrm{d}\mathbf{U}(t)}{\mathrm{d} t} = \sigma (\mathcal{F}(\mathbf{U}(t),t;\mathbf{W} ))-\gamma \mathbf{U}(t) - \alpha\mathbf{Y}(t), \\[1ex]
		\dfrac{\mathrm{d}\mathbf{Y}(t)}{\mathrm{d} t} = \mathbf{U}(t).
	\end{cases}
\end{equation}
\begin{equation}
	\begin{cases}
		D_t^\beta\mathbf{U}(t)= \sigma (\mathcal{F}(\mathbf{U}(t),t;\mathbf{W} ))-\gamma \mathbf{U}(t) - \alpha\mathbf{Y}(t), \\[1ex]
		D_t^\beta\mathbf{Y}(t)= \mathbf{U}(t).
	\end{cases}
\end{equation}
respectively, where $\zeta$ is an activation function, $\gamma$ and $\alpha$ are hyperparameters.

The integer derivatives are obtained by automatic differentiation, and the fractional derivatives are computed using the fractional Adams–Bashforth scheme~\cite{diethelm2024detailed}.

\subsubsection{Details of stability module} 

The ICNN architecture is specifically designed to ensure convexity of the output with respect to the input variables, making it particularly suitable for learning Lyapunov functions. It is implemented in this work as a two-layer feedforward neural network, which is composed of linear transformations, smooth version of ReLU activation functions, and skip connections as formalized in Eqn.(\ref{eqn:icnn}). The Lyapunov stability module is built upon the ICNN to enforce the necessary conditions for a valid Lyapunov function, combining the ICNN output with a stabilization term to ensure positive definiteness. The module constructs a Lyapunov candidate and then maps the base system into a Lyapunov stable half-space.

\subsubsection{The equilibrium-separating classification layer} 

While stable equilibrium points can provide local robustness, their direct application to graph node classification faces a critical challenge. That is, equilibrium points corresponding to different classes may lie in close proximity, resulting in small and overlapping stability neighborhoods. This phenomenon may potentially diminish adversarial robustness, where perturbations can possibly push input across the decision boundary. To mitigate this problem, we augment the fractional Lyapunov stability module with an equilibrium-separating layer, which explicitly maximizes the separation between equilibriums of different classes while preserving stability guarantees. Let $\mathbf{U}^{\prime}$ denote the output of the Lyapunov stability module. To enforce separation between class-wise equilibrium points, we employ a fully connected layer to output the parametrized form of $\mathbf{U}^{\prime}$, ensuring orthogonality via the constraint$(\mathbf{U}^{\prime})^\top\mathbf{U}^{\prime}=\mathbf{I}$.

\subsubsection{The training process}

The training process for Lyapunov stable GNNs contains two stages. In the first stage, we train the base dynamical system until convergence. To determine whether the base dynamical system has reached an equilibrium state during the training, we monitor the training accuracy over successive epochs and declare convergence to an equilibrium when the fluctuations in accuracy fall below a predefined small threshold. In the second stage, we focus on training the Lyapunov stability module, mainly updating the parameters of the ICNN and the equilibrium-separating classification layer.

\subsection{Certified Robustness under Perturbation}

We investigate the impact of initial adversarial perturbations on the model's behavior. Our analysis begins by establishing the robustness of the integer-order case, subsequently extending these findings to the fractional-order framework. It should be noted that in FROND~\cite{kang2024coupling}, the fractional extension of numerical stability analysis is applied to compare the numerical results of the model before and after adversarial perturbations, focusing on whether the accumulation of errors remains controlled when numerically solving FDEs. The difference between the two results is bounded by an upper bound, leading to the conclusion that FROND exhibits inherent robustness against perturbations. For Lyapunov stable GNNs, the research problem revolves around whether the solutions of the continuous integer- and fractional-order dynamical systems deviate from the original solutions or converge to equilibrium points under minor adversarial perturbations. Here, we show that when the node features at the initial time are subjected to initial perturbations, the systems can exhibit robustness against perturbations.

\begin{prop}[]
\label{prop1}
The equilibrium points of the projected integer-order dynamical system defined in Eqn.(\ref{dym:int}) are exponentially stable.
\end{prop}
\begin{proof}
From Eqn.(\ref{eqn:sigma}), it can be seen that when $x$ takes values in $[0,\infty)$, the activation function $\sigma$ is linear, whereas when $x$ takes values in $(0,d)$, $\sigma$ is quadratic. Therefore, the function $V$ has both upper and lower bounds. That is, $\forall \alpha>0$, assuming $\mathbf{U} \in \{\mathbf{U}: \| \mathbf{U} \| \leq \alpha\}$, there exists a positive constant $K(\alpha)$ such that 
\begin{equation*}
    \|\mathbf{U}\|^2 \leq V(t,\mathbf{U}(t)) \leq K(\alpha) \|\mathbf{U}\|^2,
\end{equation*}
that is, the first condition of Theorem~\ref{thm:int}.

Since the projected system defined in Eqn.(\ref{dym:int}) satisfies condition (4) of Theorem~\ref{thm:int}, we have,
\begin{equation*}
\frac{d V(t,\mathbf{U}(t))}{V(t,\mathbf{U}(t))} \leq -c d t.
\end{equation*}

Integrating both sides of the above equation yields
\begin{equation*}
V(t,\mathbf{U}(t)) \leq V(0,\mathbf{U}(0)) e^{-ct} \quad(t \geq 0).
\end{equation*}

Therefore, we can further obtain
\begin{equation*}
\|\mathbf{U}\|^2 \leq V(0,\mathbf{U}(0)) \leq K(\alpha) \|\mathbf{U}(0)\|^2 e^{-ct}.
\end{equation*}

This implies $\|\mathbf{U}\| \leq K(\alpha)|\mathbf{U}(0)\| e^{-ct/2}$. Hence, the equilibrium points $\mathbf{U}^{\star}$ are exponentially stable.
\end{proof}

\begin{prop}\label{prop2}
The equilibrium points of the projected fractional-order dynamical system defined in Eqn.(\ref{dym:frac}) are Mittag-Leffler stable.
\end{prop}
\begin{proof}
Let $a=2$ and $b=1$, the function $V$ satisfies the first condition of Theorem~\ref{thm:frac}:
\begin{equation} \label{eqn:condition1}
     \alpha_{1}\|\mathbf{U}\|^2\leq V(t,\mathbf{U}(t))\leq  \alpha_{2}\|\mathbf{U}\|^{2},
\end{equation}

Since the projection operation ensure the fractional system satisfied Eqn.(\ref{eqn:lya02}), there exists a nonnegative function $M(t)$ satisfying
\begin{equation} \label{eqn:proof1}
    D_t^{\beta} V(t,\mathbf{U}(t)) + M(t) = - \frac{\alpha_3}{\alpha_2}V(t,\mathbf{U}(t)).
\end{equation}

Performing the Laplace transform of Eqn.(\ref{eqn:proof1}) gives,
\begin{equation*}
    s^{\beta}V(s) - V(0)  s^{\beta-1} +M(s)= - \frac{\alpha_3}{\alpha_2}V(s),
\end{equation*}
where $s$ represents the variable in Laplace domain.

By leveraging Lipschitz condition and the inverse Laplace transform, we can obtain the solution of Eqn.(\ref{eqn:proof1}):
\begin{equation} \label{eqn:proof2}
    \begin{aligned}
         V(t)&= V(0)E_\beta (-\frac{\alpha_3}{\alpha_2}t^{\beta})- \underbrace{M(t) \ast \left[ t^{\beta-1}E_\beta(-\frac{\alpha_3}{\alpha_2}t^{\beta})\right]}_{\geq0}\\ & \leq V(0) E_\beta (-\frac{\alpha_3}{\alpha_2}t^{\beta}).
    \end{aligned}
\end{equation}

Substituting Eqn.(\ref{eqn:proof2}) to Eqn.(\ref{eqn:condition1}), we have 
\begin{equation} \label{ml_stability}
    \| \mathbf{U}(t) \| \leq \left[ m E_\beta ( - \frac{\alpha_3}{\alpha_2} t^{\beta})\right]^{\frac{1}{2}},
\end{equation}
where $m = V(0,\mathbf{U}(0))/\alpha_1 \geq0$ and $m=0$ holds if and only if $\mathbf{U}(0)=0$, which imply the projected equilibrium points are Mittag–Leffler stable.
\end{proof}

\begin{remark}
The proposed Lyapunov stable GNNs can be viewed as special association memory networks~\cite{krotov2025modern}, where the functionality of association memory is modeled by an energy function and the equilibrium points (the local minima of the energy) serve as memories. The energy landscape of the system is shaped by the Lyapunov function, which ensures the positive definiteness and monotonic decrease over time, corresponding to the process of writing information into the network (i.e. learning). The dynamical trajectory of the system represents the memory recall (i.e. inference), where features evolve toward stable equilibrium points. From this point of view, an association occurs between the initial state of the system and its asymptotic state. The stability property ensures that small perturbations to the initial state (provided they remain within the basin of attraction) are automatically corrected by the system dynamics, aligning with the certified robustness of the proposed models.
\end{remark}

\section{Experiments}
This section elaborates on the experimental setup and comprehensive performance analysis. We select representative benchmarks tailored to the task characteristics, with standardized pre-processing to ensure reproducibility and fairness. For comparison, we adapt several state-of-the-art models to comprehensively evaluate the performance advantages of the proposed approach. In our experiments, all implementations are developed in PyTorch~\cite{paszke2017automatic} and executed on a single NVIDIA GeForce RTX 3090 GPU with 24GB of memory. The source codes are available at \url{https://anonymous.4open.science/r/test-LSGNF}.

\begin{table*}[t]
\caption{Benchmarks details and the attacks's budgets of graph injection attacks (GIA)} 
\centering
\resizebox{0.6\linewidth}{!}{
\begin{tabular}{ccccccc}
\toprule
Dataset   & \# Nodes & \# Edges & \# Features & \# Classes & \begin{tabular}[c]{@{}c@{}}max \# Nodes\\ (for GIA)\end{tabular} & \begin{tabular}[c]{@{}c@{}}max \# Edges\\ (for GIA)\end{tabular} \\ \hline
Cora      & 2708     & 5429     & 1433        & 7          & 60                                                               & 20                                                               \\ \hline
Citeseer  & 3327     & 4732     & 3703        & 6          & 90                                                               & 10                                                               \\ \hline
Computers & 18333   & 81894   & 6805       & 15         & 300                                                              & 150                                                              \\ \hline
PubMed    & 19717    & 44338    & 500         & 3          & 200                                                              & 100                                                              \\ \bottomrule
\end{tabular}}
\label{tab:01}
\end{table*}

\begin{table*}[t]
\caption{Node classification accuracy (\%) on graph injection, evasion, non-targeted, black-box attack in \textit{INDUCTIVE} learning. Results under adversarial attacks that surpass the competing methods are \textbf{bold}. Second best results are with the \underline{underline}}
\resizebox{\linewidth}{!}{
\setlength{\tabcolsep}{2pt}
\begin{tabular}{c|c|cccc|cccc|cccc|c}
\toprule
Dataset   & Attack  & GRAND  & IL-GRAND & F-GRAND & FL-GRAND & GBel  & IL-GBel &F-GBel & FL-GBel & GCON  & IL-GCON & F-GCON & FL-GCON & HANG\\ \hline
\multirow{4}{*}{Cora} 
&clean   &82.24±1.82 &88.88±0.97 &86.44±0.31 &88.95±0.39  &79.07±0.46 &85.73±0.73 &77.55±0.79  &86.74±0.93   &83.10±0.63  &86.97±0.60 &82.42±0.89  &87.10±0.80 & 87.13±0.86\\
&PGD  &36.80±1.86  &79.01±1.24 &56.38±6.39  &80.76±0.33     &63.93±3.88  &\underline{80.86±0.73}  &69.50±2.83  &\textbf{82.60±0.67} &48.38±2.44   &72.71±0.38  &56.70±4.36  &74.28±0.69    &  78.37±1.84 \\
&TDGIA   &40.0±3.52  &73.41±1.17  &54.88±6.72  &78.06±0.39  & 53.22±2.95  &73.67±1.21  &56.94±1.82  &\textbf{81.87±0.60}  &46.43±2.82 & 60.46 ±3.85   &54.24±2.54     &77.06±0.73      &  \underline{79.76±0.99} \\
& MetaGIA & 37.89±1.56   &74.64±1.64  & 53.36±5.31 &   78.69±0.36      &66.74±3.23   &\underline{82.20±1.66}    &71.98±1.32   &   \textbf{82.84±0.56}  &52.21±2.71  &61.56±1.51   &63.97±2.09      & 67.74±0.92      & 77.48±1.02\\ \hline
\multirow{4}{*}{Citeseer}  
&clean   &72.52±0.73  &75.47±0.64  &71.91±0.43  &75.66±0.48  &74.75±0.28 &74.58±0.70  & 71.09±0.30  &75.03±0.47  &72.07±0.93  &63.08±1.02  &73.50±0.43 &74.61±0.80  & 74.11±0.62 \\
& PGD    &42.20±2.77  &71.62±1.39  &61.26±1.23   &72.04±1.23  & 47.73±5.87  &\underline{72.79±0.52}   &60.78±2.37  &\textbf{73.22±1.23}     &37.71±7.0  &61.03±0.95   &54.47±1.0  & 61.52±0.77   & 72.31±1.16 \\
& TDGIA   &30.02±1.33    &71.65±0.64  & 50.74±1.20  & \textbf{73.14±0.47} &47.88±1.83  &71.68±1.33   &65.52±0.55  &\underline{72.17±0.79}     & 30.93±3.00 &60.19±0.83  &54.71±1.69    &64.08±0.98    & 72.12±0.52 \\
& MetaGIA & 30.42±1.87 &71.98±1.32  &55.50±1.72  &72.17±0.97     &39.13±1.19    &72.28±0.74
&60.85±1.88  &\textbf{73.31±0.59}   &29.09±2.01  &66.93±0.84  & 48.82±3.27     &  70.52±0.94     &  \underline{72.92±0.66}\\ \hline
\multirow{4}{*}{Computers} 
& clean   &92.53±0.34 &90.43±0.26  &92.61±0.20    &90.44±0.13 &88.12±0.33    &87.48±0.55   &88.02±0.24  &87.98±0.51    &91.30±0.20  &91.44±0.25    &91.86±0.38   &91.63±0.31  &  -  \\
&PGD     & 70.45±11.03  &89.89±0.30  & 89.90±1.33   & 90.02±0.33 & 87.38±0.37 &87.32±0.57     &87.60±0.33  & 87.66±0.60    &81.28±7.99  & 91.11±0.28     &\underline{91.36±0.74}    &\textbf{91.55±0.27}      &  -  \\
&TDGIA   &65.45±14.30 & \textbf{91.13±0.21}  &84.71±1.52  &90.44±0.19         & 87.67±0.40  & 87.69±0.49     
&87.81±0.28   &87.91±0.48    &68.70±15.67   &\underline{90.75±0.25} &90.45±0.71    &90.63±0.31     & - \\
& MetaGIA & 70.01±9.32  &89.71±0.43  &87.50±3.17   &  90.23±0.36          &87.77±0.22 &87.94±0.56       
&87.37±0.23   &   87.97±0.55    &82.43±8.42 & \textbf{91.75±0.25}  & 90.51±0.88      &\underline{91.26±0.57}      & - \\ \hline
\multirow{4}{*}{Pubmed}   
& clean   & 88.44±0.34  &86.89±0.65  &88.39±0.47  & 87.43±0.42  & 88.18±1.89 &90.25 ±0.20
&89.51±0.12    &  93.78±0.13     & 88.09±0.32 &86.73±0.25  & 90.30±0.11    &  87.27±0.74    &  89.93±0.27   \\
& PGD     & 44.61±2.78 &63.01±0.40 & 59.62±11.66 &   61.74±0.31  & 67.81±12.23   &  69.06±0.26
& 72.04±2.08   &  \underline{74.18±0.09}    & 45.85±1.97 &61.49±0.29   & 51.16±6.04     & 61.77±0.22       &   \textbf{81.81±1.94}\\
& TDGIA   & 46.26±1.32  &60.10±0.94  & 54.31±2.38 & 60.64±1.10         & 68.66±10.64  &\underline{73.15±0.34}  
&72.49±0.83   &  72.93±0.18   & 45.57±2.02 &69.81±0.33   & 55.50±4.03     &  70.21±0.32     &  \textbf{86.62±1.05}\\
& MetaGIA & 44.07±2.11 &66.93±0.40 &61.62±9.05  &68.68±1.24      &64.64±9.70   & 76.13±0.32  
& 79.16±0.87   &  \underline{80.20±0.19}     & 45.81±2.81  &56.63±0.25    & 52.03±5.53     &   56.11±0.69    &  \textbf{87.58±0.75}   \\ \bottomrule
\end{tabular}}
\label{tab:black}
\end{table*}

\begin{table*}[t]
\caption{Classification accuracy (\%) under modification, poisoning, non-targeted attack (Metattack) in \textit{TRANSDUCTIVE} learning}
\resizebox{\linewidth}{!}{
\setlength{\tabcolsep}{2pt}
\begin{tabular}{c|c|cccc|cccc|cccc|c}
\toprule
Dataset                    & Ptb(\%)  & GRAND &IL-GRAND & F-GRAND & FL-GRAND & GBel &IL-GBel & F-GBel & FL-GBel & GCON &IL-GCON & F-GCON & FL-GCON & HANG\\ \hline
\multirow{6}{*}{Cora} 
& 0  & 82.24±1.82 &78.56±0.39 & 81.25±0.89  & 81.75±0.75  & 80.28±0.87 &79.32±0.49
& 79.05±0.73   & 78.40 ±0.18    & 83.10±0.79  &80.67±0.09  & 80.91±0.54   &  80.00±0.06 &- \\
& 5    & \underline{78.97±0.49} &77.04±0.49  & 78.84±0.57    & \textbf{81.04±0.12}  & 77.70±0.66  &77.53±0.47   &76.10±0.74   & 76.26±0.55      & 77.90±1.14  &77.11±0.21  & 77.80±0.44     &  76.20±0.33  &-  \\
& 10  & 75.02±1.25  &76.58±0.15  & \underline{76.61±0.15}  & \textbf{80.08±0.29} & 74.30±0.88  &74.19±0.21  & 74.03±0.47  & 75.15±0.29  &72.53±1.08 &75.55±0.20    & 74.63±1.42     &  75.20±0.52  &  -    \\
& 15 & 71.43±1.09 &\underline{76.48±0.17} & 73.42±0.97 &   \textbf{79.83±0.27}  &72.14±0.69 &72.35±0.70     &73.01±0.75  &74.31±0.63  &69.83±0.68  &73.99±0.27  & 73.01±0.78   &   73.22±0.55   &   -  \\
& 20 & 60.53±1.99 &\underline{76.48±0.20} & 69.27±2.10 & \textbf{78.20±0.43}  & 65.41±0.99  &72.91±0.48    & 69.35±1.23  & 73.64±0.09   & 57.28±1.62  &72.48±0.22    & 69.23±1.35      & 72.99±0.39  & -\\ 
& 25 & 55.26±2.14 &\underline{75.99±0.20}  & 64.47±1.83  & \textbf{76.70±0.60}  & 62.31±1.13  &71.72±0.33    & 67.63±0.93  &72.72±0.22  &53.17±1.52  &70.02±0.73   & 65.27±1.33     & 72.17±0.61  & \\ \hline
\multirow{6}{*}{Citeseer} 
& 0  & 71.50±1.10 &71.94±0.31  & 71.37±1.34  & 71.65±0.49 & 71.37±1.34 &71.05±0.43
&  68.90±1.15  & 71.06±0.22    & 70.48±1.18   &71.42±0.38 & 71.49±0.71    & 70.46±0.14 & -\\
& 5    & 71.04±1.15 & \textbf{71.61±0.43} & \underline{71.47±0.9}6  & 69.89±0.16   & 68.45±1.02  &70.24±0.65   & 68.36±0.93  &  70.66±0.14     & 69.75±1.63 &70.33±0.62   & 70.77±1.15    &  70.89±0.45  & -\\
& 10  & 68.88±0.60 &\textbf{71.27±0.83}  & 69.76±0.71   & 69.78±0.12  & 69.54±0.82 &69.12±0.10    & 66.72±1.31   & 70.20±0.38     & 67.40±1.78 &70.51±0.19    & 69.54±0.82     & \underline{70.93±0.24}   &    -   \\
& 15 &66.35±1.37 &69.40±0.57 & 67.94±1.42   & 68.43±0.22       & 63.63±1.67 & 68.73 ±0.23    &63.56±1.95 &   69.98±0.15    & 65.78±1.97 &\underline{70.00±0.83} &67.37±0.87      &  \textbf{70.30±0.26}   &   -   \\
& 20 & 58.71±1.42 &68.62±0.32 &64.18±0.93 & 68.03±0.13       & 58.90±0.84  & 66.56 ±0.32  & 63.38±0.96   & 68.37±0.15      & 56.79±1.46  &\underline{70.03±0.07}  & 66.52±0.68     &  \textbf{70.23±0.21}  & - \\ 
& 25 & 60.15±1.37 &65.30±0.20 & 65.46±1.12  &  67.23±0.12        &61.24±1.28  &65.18±0.22     & 64.60±0.48   &  67.71±0.17    & 57.30±1.38  & \textbf{69.77±0.10} & 66.72±1.12    &  \underline{69.20±0.23} & -\\ \hline
\multirow{6}{*}{Pubmed} 
& 0  & 85.06±0.26 &81.61±0.04 & 82.28±0.10  & 85.67±0.08  & 84.02±0.26
&82.17±0.03 & 86.15±0.14   & 86.18±0.01     & 84.65±0.13 &83.52±0.03   & 87.19±0.09   & 87.23±0.03  &85.08±0.20 \\ 
& 5    & 84.11±0.30 &79.24±0.04 &81.27±0.21   & 85.56±0.03      & 83.91±0.26   &82.13±0.12   & 85.91±0.14   &   85.93±0.02    & 83.06±0.22  &79.58±0.03  & \underline{86.38±0.19}     & \textbf{86.58±0.12} & 85.08±0.18 \\ 
& 10  &84.24±0.18  &76.11±0.11  & 80.64±0.19  & 85.31±0.06 & 84.62±0.26 &81.55±0.09    & 85.90±0.20   &   85.48±0.01    & 82.25±0.12 &76.35±0.02  & \underline{85.90±0.20}     & \textbf{86.13±0.15}  &   85.17±0.23 \\
& 15 & 83.74±0.34 &72.64±0.07 & 79.95±0.25  & 85.09±0.02   & 84.83±0.20  & 81.06±0.06   & 85.45±0.33  & 85.34±0.28    & 81.26±0.33 &72.92±0.12  & \underline{85.78±0.37}    &   \textbf{85.90±0.01} &    85.0±0.22   \\
& 20 & 83.58±0.20 &70.38±0.08  & 79.96±0.21 & 84.54±0.09      & 84.89±0.45 &81.29±0.04    & 85.44±0.16   & 85.03±0.41    & 81.58±0.41 &70.31 ±0.02   &\underline{85.55±0.02}     &   \textbf{85.74±0.19} & 85.20±0.19 \\ 
& 25 & 83.66±0.25 &68.15±0.03 & 79.95±0.26 & 84.22±0.03    &85.07±0.15  & 80.95±0.10     & 85.19±0.14   &  85.15±0.11    &80.75±0.32 &68.06±0.05    &    \underline{85.37±0.23}  & \textbf{85.83±0.17}  & 85.06±0.17\\ \bottomrule
\end{tabular}}
\label{tab:ptb}
\end{table*}

\begin{table*}[t]
\caption{Node classification accuracy (\%) on graph injection, evasion, non-targeted, \textbf{white-box} attack in inductive learning. AT is the abbreviation for adversarial training.} 
\resizebox{\linewidth}{!}{
\begin{tabular}{c|c|cccc|cccc|cccc|c}
\toprule
Dataset    & Attack  & GRAND &  IL-GRAND &FL-GRAND & FL-GRAND+AT & GBel & IL-GBel & FL-GBel  & FL-GBel+AT &  GCON &IL-GCON & FL-GCON & FL-GCON+AT & HANG\\ \hline
\multirow{3}{*}{Cora} 
& clean   & 87.53±0.59  &89.12±0.23  & 89.58±0.60 &  89.24±0.84& 86.13±0.51 &86.47±0.58 & 87.03±1.06 &  85.79±0.84    &84.69±1.10  &87.97±0.23   &   85.53±0.96    &87.13±0.86 &  88.31±0.48 \\
& PGD     &36.02±4.09 &70.01±0.88 &\underline{72.22±0.64}   &\textbf{72.75±0.91}  & 37.16±1.69 &59.87±2.26 & 59.43±2.43  & 61.35±1.96 &35.83±0.71 &53.45±1.01   &52.96±1.19   & 53.90±1.24    & 67.69±3.84     \\
& TDGIA   &14.72±1.97 &58.55±1.34 & 56.65±2.91  & 59.32±1.84  & 15.46±1.98 &66.98±2.14 &\underline{67.54±2.01}&\textbf{69.12±2.46}    & 33.05±1.09 &55.14±1.96 & 54.44±1.40     & 55.90±2.95   & 64.54±3.95    \\ \hline
\multirow{3}{*}{Citeseer} 
& clean   & 74.98±0.45 & 74.66±0.63  & 75.03±0.66  &  74.88±0.59 & 69.62±0.56
&74.56±1.17 &74.69±1.19  &73.04±0.51    & 74.84±0.49   &64.09±2.48   &  74.36±1.04     &73.51±2.00 & 74.11±0.62  \\
& PGD     &38.57±1.94 &66.58±0.77  &66.79±0.98  & \underline{67.97±1.16}  &32.24±1.21 & 67.02±1.25 &66.21±1.50 & \textbf{70.46±2.13}   & 42.78±1.54 &52.00±4.27   &58.23±4.69    &65.23±1.19     &67.54±1.52   \\
    & TDGIA   & 30.11±1.43 &54.91±0.66  &54.31±0.86   &55.37±0.45  & 16.26±1.20 &75.35±0.58 & \underline{75.63±0.77}& \textbf{76.09±0.79}   & 33.55±1.10 &66.14±1.38 &66.46±0.99      &67.01±2.04   &   63.29±3.15   \\ \hline
\end{tabular}}
\label{tab:white}
\end{table*}

\subsection{The Experimental Setup}

\subsubsection{Datasets} We evaluate the family of Lyapunov stable GNNs, namely the IL-GNNs and FL-GNNs, on four datasets, including different graph-structures, such as citation networks and co-purchasing graphs. Citation networks, such as Cora~\cite{McCallum2000AutomatingTC}, Citeseer~\cite{Sen2008CollectiveCI}, and Pubmed~\cite{Namata2012QuerydrivenAS}, represent academic articles as nodes and citation relationships as edges. The features of each node typically include text embeddings or topic distributions of the article, while the node labels correspond to the classification of the article. Co-purchasing graphs, such as Computer~\cite{McAuley2015ImageBasedRO}, model products as nodes, with edges indicating frequently co-purchased item pairs. These datasets cover multiple fields and vary in scale, offering a comprehensive foundation for the proposed models.

\subsubsection{Adversarial attacks} To ensure a thorough evaluation of graph adversarial attacks, we perform two kinds of attack scenarios: graph injection attacks (GIA) in the inference phase and graph modification attacks (GMA) in the training phase. 

For GIA, we adapt three strategies: (1) PGD-GIA~\cite{chen2022hao}, which randomly injects nodes and modifies their features using projected gradient descent; (2) TDGIA~\cite{zou2021tdgia}, a topology-aware method that exploits structural deficiencies to guide edge creation while optimizing node features via a loss function; and (3) MetaGIA~\cite{chen2022hao}, which dynamically optimizes both node features and graph structure through iterative gradient-based updates. The selected attack parameters maintain uniformity with the criteria established in the literature. To mitigate potential bias in adversarial robustness evaluation, we adopt a data filtering strategy proposed by Zheng et al.~\cite{Zheng2021GraphRB}, which removes nodes with the lowest 5\% degrees and the highest 5\% degrees, ensuring a balanced testbed for adversarial scenarios. The benchmarks and attack parameters employed in the paper are detailed in Table~\ref{tab:01}.

For GMA, the experiments strictly adhere to the established benchmark configuration~\cite{jin2020graph}. The perturbation rate is increased from 0\% (representing the clean graph) to 25\% in increments of 5\%, enabling a consistent comparison across different attack intensities.

\begin{table*}[t]
\small 
\centering
\caption{Classification accuracy (\%) of the IL- and FL-GNNs \textbf{without equilibrium-separating (ES) layer} on graph injection, evasion, non-targeted attack in inductive learning}
\begin{tabular}{c|c|ccc|ccc}
\toprule
\thead{Dataset}        & \thead{Attack}  & \thead{IL-GRAND \\ (w/o ES layer)} & \thead{IL-GBel \\ (w/o ES layer)} & \thead{IL-GCON \\ (w/o ES layer)} & \thead{FL-GRAND \\ (w/o ES layer)} & \thead{FL-GBel  \\ (w/o ES layer) } & \thead{FL-GCON  \\ (w/o ES layer)} \\ \hline
\multirow{4}{*}{Cora} 
& clean   &87.77±1.01 &86.29±1.07 &85.47±0.75 & 88.18±1.38    & 85.98±1.53  &  84.33±0.46     \\
& PGD   &78.88±0.92 &79.75±0.46 &70.27±0.72 & 78.84±1.24   & 78.00±0.87  &   70.18±0.89       \\
& TDGIA  &72.40±0.83 &74.06±1.27 &69.62±0.96 & 74.76±1.18  & 74.78±0.99 &   76.43±1.31   \\
& MetaGIA &76.44±0.89 &76.20±0.87 &61.50±1.51 & 76.98±0.58  & 83.02±1.54 & 64.08±1.55      \\ \hline
\multirow{4}{*}{Citeseer}  
& clean   &73.39±0.85 &72.93±0.63 &66.70±1.68  &73.26±1.08  & 73.52±1.13 & 63.47±1.18 \\
& PGD    &70.79±0.80 &71.98±1.41 &60.84±0.94   &71.21±0.65  & 72.29±1.37 & 61.12±0.36   \\
& TDGIA   &71.76±0.59 &71.28±0.76 &63.95±1.68   &72.26±0.72    &71.73±1.20 & 62.21±1.98\\
& MetaGIA  &71.03±0.79 &71.66±0.93 &63.89±1.13  &70.90±0.48   & 72.45±0.50 &  65.11±1.64     \\ \hline
\multirow{4}{*}{Pubmed}    
& clean   &87.87±0.53 &87.14±0.29 &85.12±0.42  &  86.49±0.28   &87.10±0.32  & 84.48±0.49        \\
& PGD     &64.79±0.26 & 73.27±0.27 &61.88±0.91  &61.99±0.91  &73.19±0.30 & 60.25±0.72        \\
& TDGIA   &60.21±0.54 &73.43±0.18 &71.56±0.35  & 59.40±0.27    &72.80±0.19  &  69.58±0.46          \\
& MetaGIA &60.42±0.77 &74.12±0.33 &50.83±0.28  & 65.78±0.19  & 79.76±1.05 &  53.03±3.62     \\ \bottomrule
\end{tabular}
\end{table*}
\begin{figure}[t]
\centering
\includegraphics[width=0.51\textwidth]{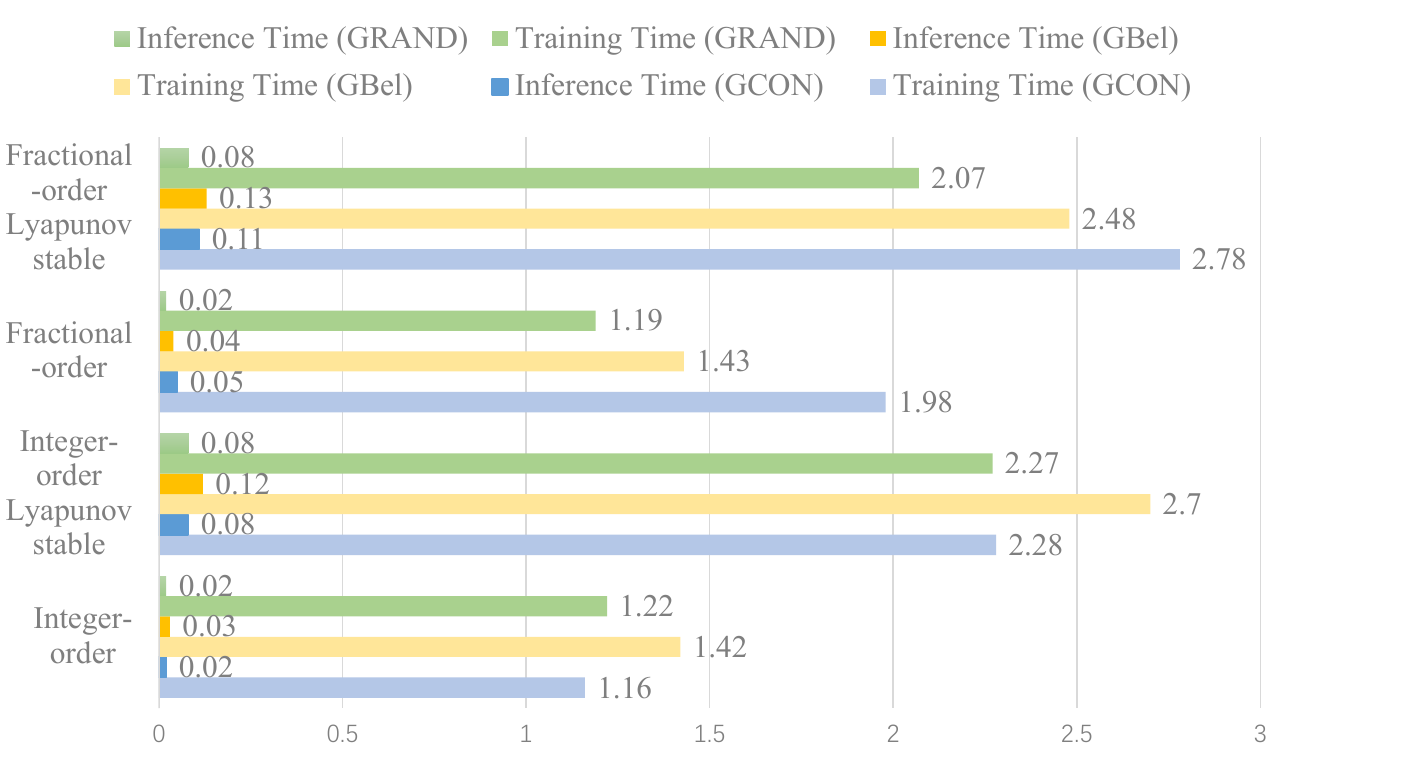}
\caption{Training time (s/epoch) and inference time (s) on the Cora dataset with $t\in[0,10]$ and step size of 1.}
\label{fig:architecuture}
\end{figure}

\begin{table}[t]
\small 
\caption{Node classification accuracy (\%) on graph injection, evasion, non-targeted attack in inductive learning \textbf{under different attack amplitudes} on the Cora dataset. \# N./E. means the number of nodes and edges injected.}
\begin{tabular}{c|c|c|c|c}
\hline
Attack                   & \# N./E. & FL-GRAND & FL-GBel & FL-GCON \\ \hline
\multirow{4}{*}{PGD}     & 80/40    & 73.76±0.52         &81.67±0.82       &   68.14±0.76     \\
                         & 120/80            &  65.39±1.06        &  77.18±0.99     &  61.44±1.79      \\
                         & 160/120           & 55.63±2.49         & 73.56±0.94        &56.49±1.37     \\
                         & 200/160           & 52.02±4.91        & 66.41±0.74       & 55.23±2.09 \\ \hline
\multirow{4}{*}{TDGIA}   &  80/40 &75.20±1.25    &79.66±0.85   &72.71±0.66       \\
                         &  120/80  &68.54±0.66         &73.58±0.95 & 59.34±3.99   \\
                         &  160/120     &60.24±2.61       &71.48±1.28 &52.34±2.81      \\
                         &   200/160    &55.49±1.87        & 64.30±2.44       & 50.24±4.24       \\ \hline
\multirow{4}{*}{MetaGIA} &  80/40     &  72.58±0.66  & 71.61±1.73   &  69.35±0.60       \\
                         &   120/80   & 63.77±0.97   & 67.28±0.82   &  61.03±1.49      \\
                         &  160/120     &59.01±2.07         & 65.61±1.89       &  59.35±1.39     \\
                         &  200/160     & 50.86±1.03        &  58.62±2.46      &  57.76±1.97     \\ \hline
\end{tabular}
\label{tab:number}
\end{table}

\subsection{The Experimental Results and Analysis}

This paper emphasizes the robustness of GNNs. We demonstrate that Lyapunov stable graph neural flows achieve significantly stronger robustness compared to vanilla integer- and fractional-order graph neural flows.

\subsubsection{Under the inductive learning settings} As shown in Table~\ref{tab:black}, both IL-GNNs and FL-GNNs achieve competitive performance on clean data and demonstrate significant robustness advantages across multiple datasets and GIA methods. For example, compared to F-GRAND on the Cora dataset, FL-GRAND not only improves the performance on clean data by average of 2.51\%, but also achieves an average increase in accuracy of 24.38\%, 23.18\%, 25.33\% under PGD, TDGIA and MetaGIA, respectively. These results show that our method can effectively mitigate the impact of adversarial perturbations. Furthermore, because the Lyapunov function is designed to be differentiable, its gradient can be computed directly using automatic differentiation. Consequently, as shown in Figure 2, the training and inference time of the proposed methods are slightly higher than that of the baselines, which is a justified trade-off given the improved resilience against adversarial attacks.

\subsubsection{Under the transductive learning settings} As shown in Table~\ref{tab:ptb}, IL-GNNs and FL-GNNs experience a smaller performance degradation under high perturbation (ptb) rates compared to baseline models. For example, on the Citeseer dataset, FL-GRAND exhibits marginally inferior accuracy to F-GRAND under small ptb rates, while surpassing its counterpart substantially with higher ptb rates; Moreover, the accuracy of GCON and F-GCON decreases by 13.18\% and 4.77\% as the ptb rate increases from 0\% to 25\%, while that of IL-GCON and FL-GCON only decreases by 1.65\% and 1.26\% in the same case, indicating that IL-GNNs and FL-GNNs are insensitive to ptb rates and further validating the effectiveness of the stability module. 

In general, the average accuracy of IL-GNNs and FL-GNNs over ten rounds of experiments outperforms the baseline models, and its standard deviation is smaller than that of the baseline models in the vast majority of robustness evaluation cases, demonstrating that the proposed models exhibits strong reliability. 

\subsubsection{Under the white-box attack settings} To further evaluate robustness, we now expand our analysis to include white-box, injection, and evasion attacks. Under white-box attack settings, adversaries possess complete knowledge of the target model, which constitutes a far more powerful threat compared to black-box scenarios. we conduct inductive learning tasks following the same setup as Table~\ref{tab:black}, with results detailed in Table~\ref{tab:white}. For instance, on the Cora dataset under the PGD attack, the accuracy of GRAND drops drastically from 87.53\% to 36.02\%, while GBel and GCON experience similar declines. In contrast, our framework maintains stronger resilience compared to its baseline models. Our findings reveal that under white-box attacks, all baseline models suffer severe performance degradation. In contrast, the FL-GNN framework maintains stronger resilience compared to baseline models.

\subsubsection{The robustness of FL-GNNs under different attack amplitudes} We evaluate the robustness of FL-GNNs against adversarial attacks with varying perturbation amplitudes, as summarized in Table~\ref{tab:number}. Our results demonstrate that FL-GNNs maintain strong robustness even under extreme attack budgets (up to 200 nodes and 160 edges perturbed), demonstrating consistently stable performance even as the attack budgets grow. This suggests inherent resilience to large-scale adversarial perturbations.

\begin{figure*}[t] 
\centering
    \begin{subfigure}[b]{0.32\linewidth}
        \centering
        \includegraphics[width=\linewidth]{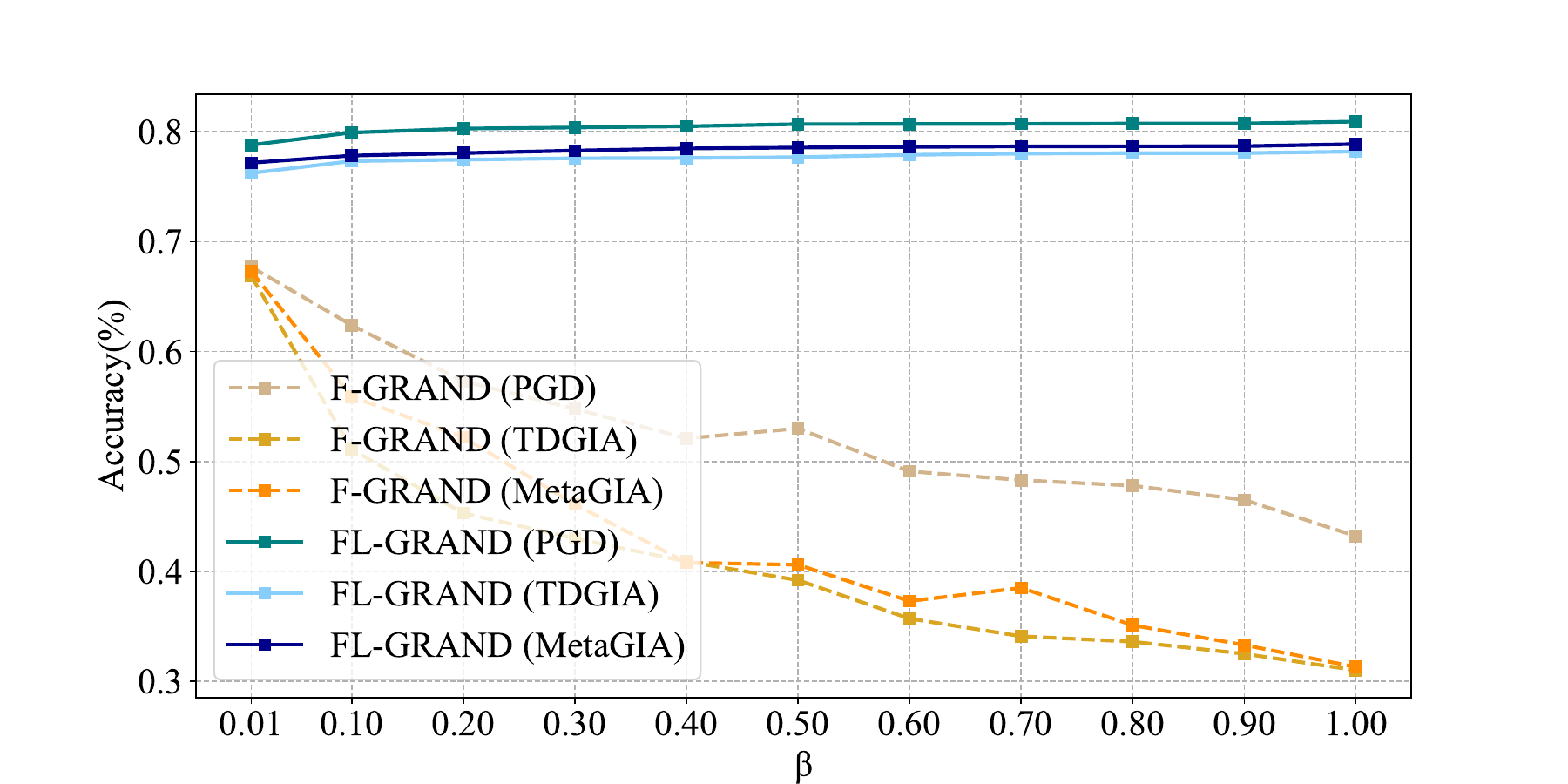}
        \caption{FL-GNN based on GRAND}
    \end{subfigure}
    \hfill
    \begin{subfigure}[b]{0.32\linewidth}
        \centering
        \includegraphics[width=\linewidth]{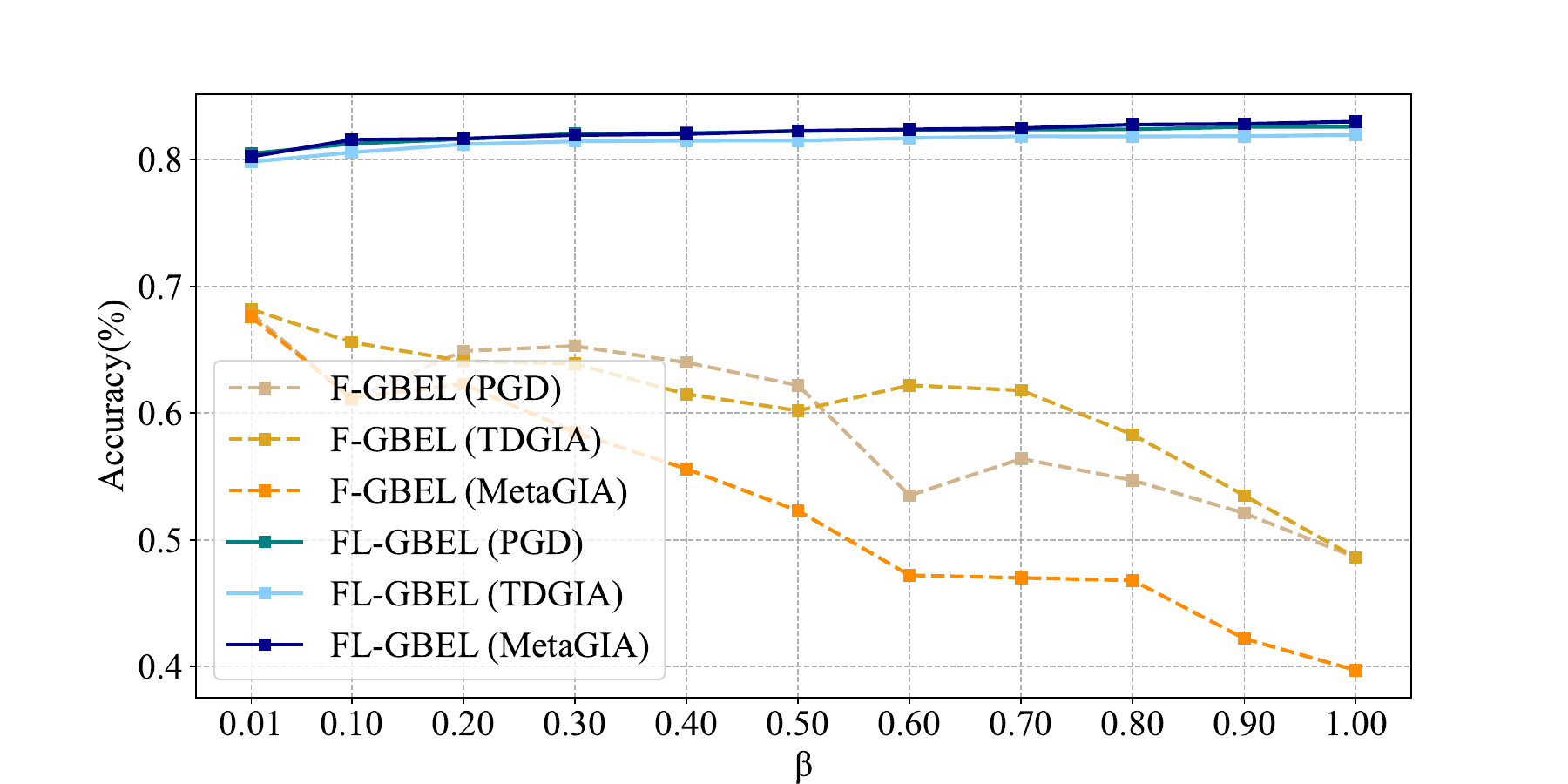}
        \caption{FL-GNN based on GBEL}
    \end{subfigure}
    \hfill
    \begin{subfigure}[b]{0.32\linewidth}
        \centering
        \includegraphics[width=\linewidth]{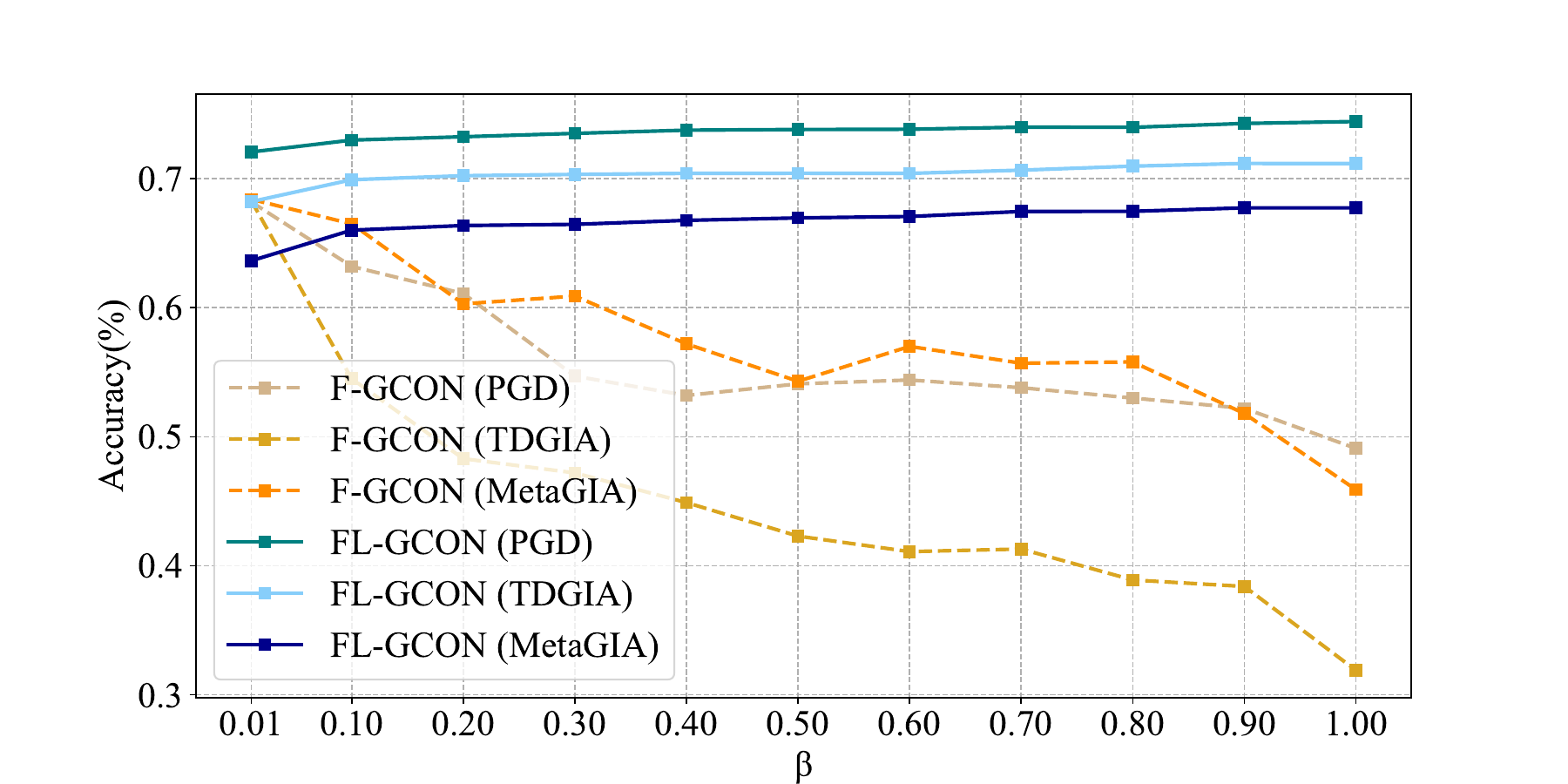}
        \caption{FL-GNN based on GCON}
    \end{subfigure}
\caption{Influence of $\beta$ on the robust test accuracy on Cora dataset.}
\label{fig:comparison}
\end{figure*}

\subsubsection{The robustness of FL-GNNs combining with graph adversarial training} We demonstrate that incorporating adversarial training (AT) into FL-GNNs can further enhance robustness. This is because adversarial training and the proposed method improve the resilience of a model from different perspectives. Specifically, the fractional Lyapunov stability module operates on the model itself, ensuring its inherent capability to resist adversarial perturbations. In contrast, AT techniques introduce generated adversarial examples into the training data to strengthen the model's resilience against adversarial attacks. The vanilla AT is essentially a min-max optimization problem:
\begin{equation} \label{eqn:02:adv_training}
\min_{\boldsymbol{\theta}} \mathbb{E}_{(\boldsymbol{x}, y) \sim \mathcal{D}}[ \max_{\boldsymbol{\rho}} \mathcal{L}(z(\boldsymbol{x} + \boldsymbol{\rho}; \boldsymbol{\theta}), y)],
\end{equation}
where $\mathcal{D}$ denotes the data distribution, $\boldsymbol{\rho}$ represents the adversarial perturbation, $\boldsymbol{x}$ a clean input sample, and $y$ its ground‑truth label. The inner maximization searches for a perturbation $\boldsymbol{\rho}$ that maximizes the classification loss $\mathcal{L}$, while the outer minimization then updates the model parameters $\boldsymbol{\theta}$ to reduce the loss on these adversarial examples, which improves the model’s robustness against such perturbations.

The experimental results demonstrate that the proposed FL-GNNs (FL-GRAND, FL-GBel, and FL-GCON) combined with AT exhibit improved robustness against white-box adversarial attacks in inductive learning. Specifically, FL-GRAND+AT achieves an accuracy boost of 4.25\% under PGD attack and 0.46\% under TDGIA attack, outperforming its non-AT counterpart on the Citeseer dataset. This highlights the effectiveness of adversarial training in further enhancing the robustness of the proposed method.

\subsubsection{Comparison with the related method}

Unlike the energy conservation-based robustification method HANG~\cite{zhao2023adversarial}, our approach employs a learnable projection mechanism to explicitly enforce energy dissipation. This guides the system toward the minimum energy state corresponding to stable equilibrium points. Empirically, the proposed models outperform HANG across most adversarial settings (with the exception of the Pubmed dataset in Table~\ref{tab:black}). These results demonstrate that our framework is better suited for handling various adversarial settings and generalizes well across standard benchmarks.

\subsection{Ablation Study}

\subsubsection{Influence of the equilibrium-separating layer} To verify the necessity of the equilibrium-separating layer, we conduct an ablation study by replacing it with a standard fully connected layer. As shown in Table 5, the absence of this component generally leads to reduced robustness compared to the results in Table 2, confirming its role in enhancing the resilience of FL-GNNs. Although we observe isolated cases where the equilibrium-separating layer could potentially degrade performance, its overall contribution to robustness outweighs these marginal drawbacks.

\subsubsection{Influence of $\beta$} We evaluate the performance of our method by adjusting $\beta$. As shown in Figure 2, our approach consistently demonstrates superior robustness and minimal fluctuations compared to FROND under three types of governing equation with different values of $\beta$. It is worth mentioning the Eqn.(\ref{ml_stability}) in the proof of Proposition 6 shows that larger $\beta$ generally leads to faster convergence of the Mittag–Leffler stability, implying quicker stabilization of the system. The incrementally increasing trend in adversarial robustness accuracy in Figure 2 indicates that larger $\beta$ yields smaller adversarial fluctuations, which further supports the alignment between theory and experiment.

\section{Conclusions}

In this paper, we proposed a model robustification framework against graph adversarial attacks by integrating Lyapunov stability theory with graph neural flow and thereby introduced a new family of Lyapunov stable GNNs. By constraining the GNN dynamics to satisfy the conditions of integer- or fractional-order Lyapunov's direct method, our approach ensures provable stability under malicious adversarial attacks. Experiments demonstrated that both integer- and fractional-order Lyapunov stable neural flow substantially outperforms the base neural flow across standard benchmarks and various attack scenarios. There are several promising directions for future research, such as the extension of the framework to other graph-related tasks, and the investigation of the interplay between fractional dynamics with more advanced theories (e.g., FTLE or finite‐time stability analysis) to broaden the theoretical underpinnings of robust GNNs.

\bibliographystyle{IEEEtran}
\bibliography{ref}{}
\newpage

\begin{IEEEbiography}[{\includegraphics[width=1in,height=1.25in,clip,keepaspectratio]{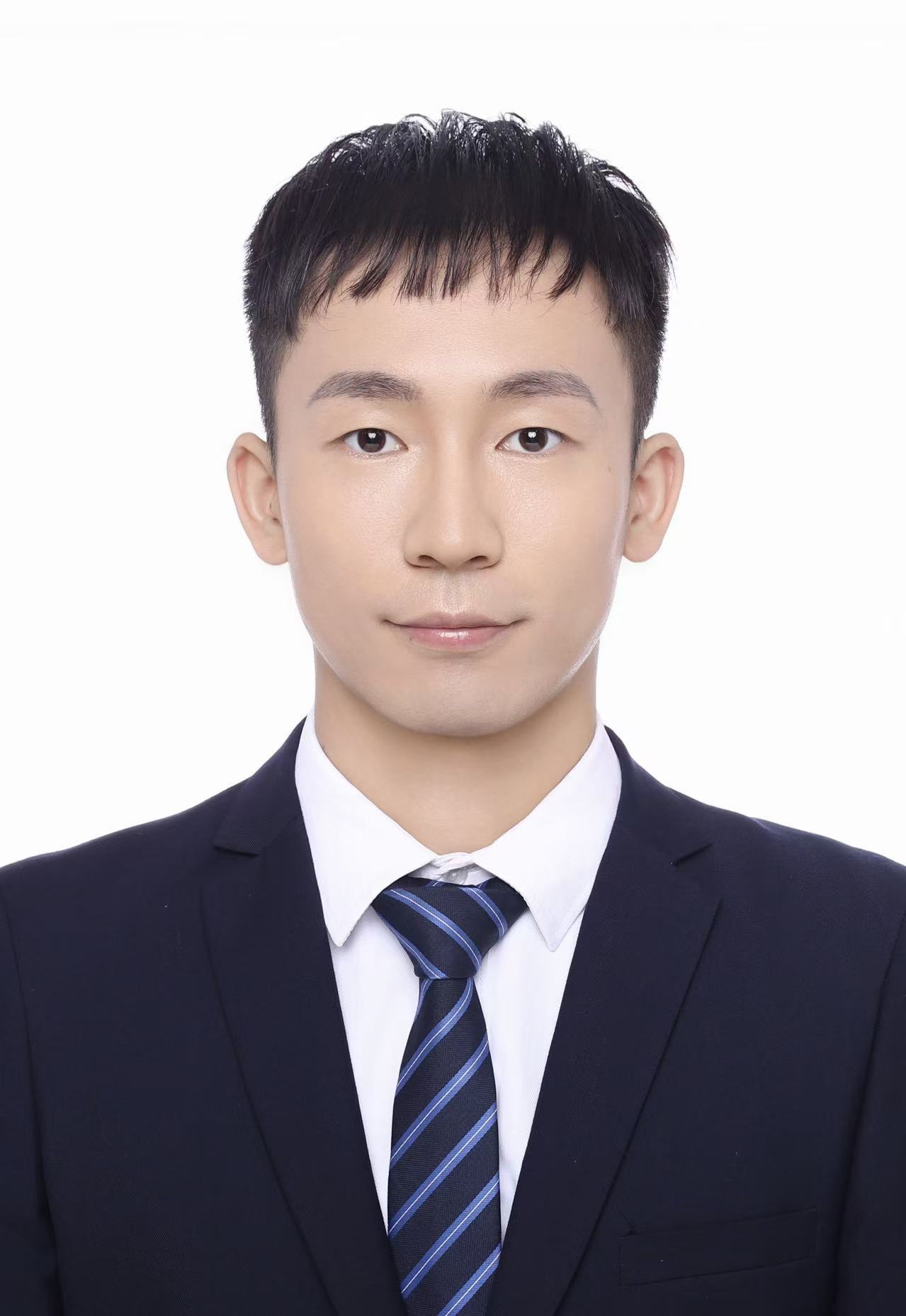}}]{Haoyu Chu} is currently a Lecturer at the School of Computer Science and Technology, China
University of Mining and Technology (CUMT). He received the Ph.D. degree from the Institute of Information Science, Beijing Jiaotong University (BJTU), in 2024. Prior to that, he was a visiting Ph.D. student at the Graduate School of Information Science and Technology, Osaka University. His research mainly focuses on scientific machine learning, physics-informed deep learning, and AI for scientific computing.
\end{IEEEbiography}

\begin{IEEEbiographynophoto}{Xiaotong Chen} received the B.E. degree from Henan University at Kaifeng and M.S. degree from Beijing Jiaotong University(BJTU), China, in 2020 and 2022. She is currently a Ph.D. candidate at the Institute of Information Science, Beijing Jiaotong University, Beijing, China. Her research interests include cross-modal retrieval and multi-modal learning.\end{IEEEbiographynophoto}
\vspace{11pt}

\begin{IEEEbiography}[{\includegraphics[width=1in,height=1.25in,clip,keepaspectratio]{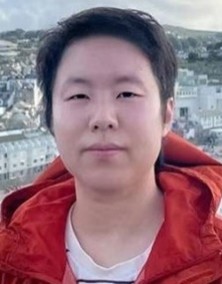}}]{Wei Zhou} (IEEE Senior Member) is an Assistant Professor at Cardiff University, United Kingdom. Previously, Wei studied and worked at other institutions such as the University of Waterloo (Canada), the National Institute of Informatics (Japan), the University of Science and Technology of China, Intel, Microsoft Research, and Alibaba Group. Dr Zhou is now an Associate Editor of IEEE Transactions on Neural Networks and Learning Systems (TNNLS), ACM Transactions on Multimedia Computing, Communications, and Applications (TOMM), and Pattern Recognition. Wei’s research interests span multimedia computing, perceptual image processing, and computational vision.
\end{IEEEbiography}

\begin{IEEEbiography}[{\includegraphics[width=1in,height=1.25in,clip,keepaspectratio]{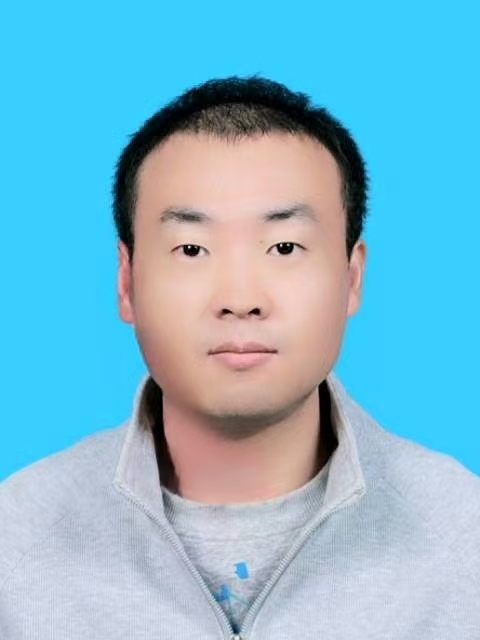}}]{Wenjun Cui} (IEEE Senior Member) is currently a Lecturer with School of Automation and Software Engineering, Shanxi University, China. He received the Ph.D. degree from Beijing Jiaotong University. His research mainly focuses on machine learning and differential equations.
\end{IEEEbiography}

 \begin{IEEEbiography}[{\includegraphics[width=1in,height=1.25in,clip,keepaspectratio]{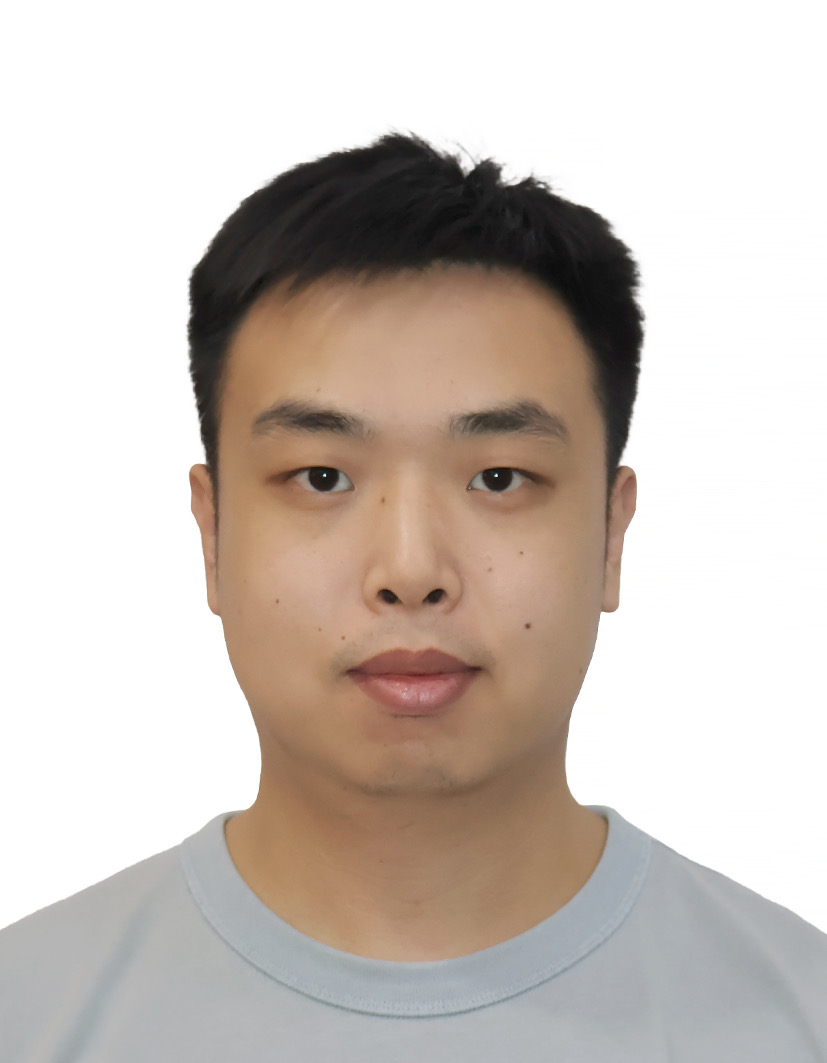}}]{Kai Zhao}
received the Ph.D. degree from the School of Electrical and Electronic Engineering, Nanyang Technological University, Singapore, in 2025. He received the B.Eng. degree in electrical engineering and automation from the Huazhong University of Science and Technology, China, in 2017, and the M.Sc. degree in electrical engineering from the National University of Singapore in 2019. He is currently with TikTok, Singapore. His research interests include graph learning and adversarial robustness in machine learning.
\end{IEEEbiography}

\begin{IEEEbiography}[{\includegraphics[width=1in,height=1.25in,clip,keepaspectratio]{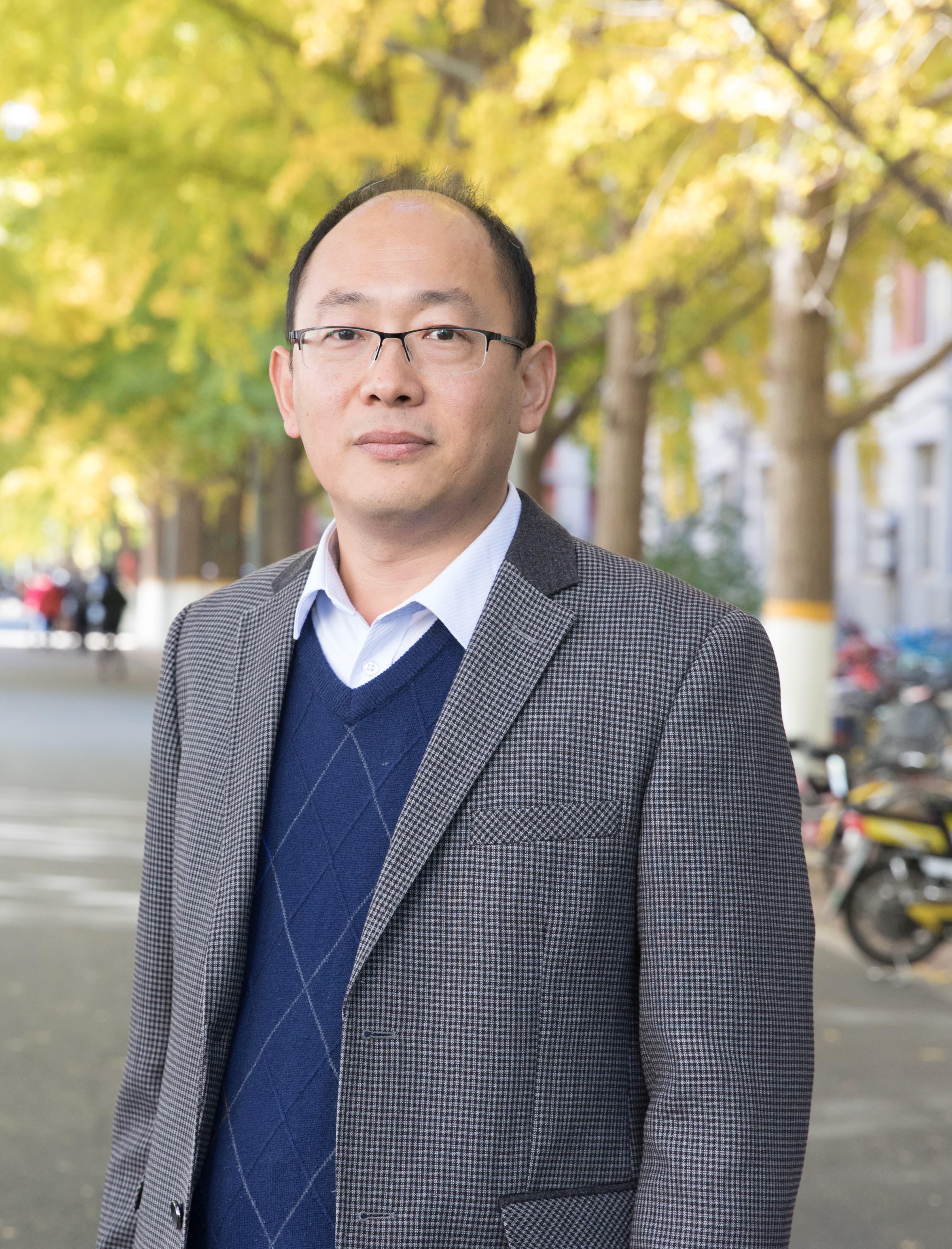}}]{Shikui Wei} is a Professor at Beijing Jiaotong University, China. He received his Ph.D. in Engineering from Beijing Jiaotong University in 2010. From 2008 to 2010, he was a joint Ph.D. student at Nanyang Technological University, Singapore. He then worked as a Postdoctoral Fellow at Nanyang Technological University from 2010 to 2011. In 2012, he visited Microsoft Research Asia for six months. In 2015, he joined the China Mobile Research Institute as a Deputy Principal Researcher. His research interests include cross-modal intelligent fusion and machine learning.
\end{IEEEbiography}

\begin{IEEEbiography}[{\includegraphics[width=1in,height=1.25in,clip,keepaspectratio]{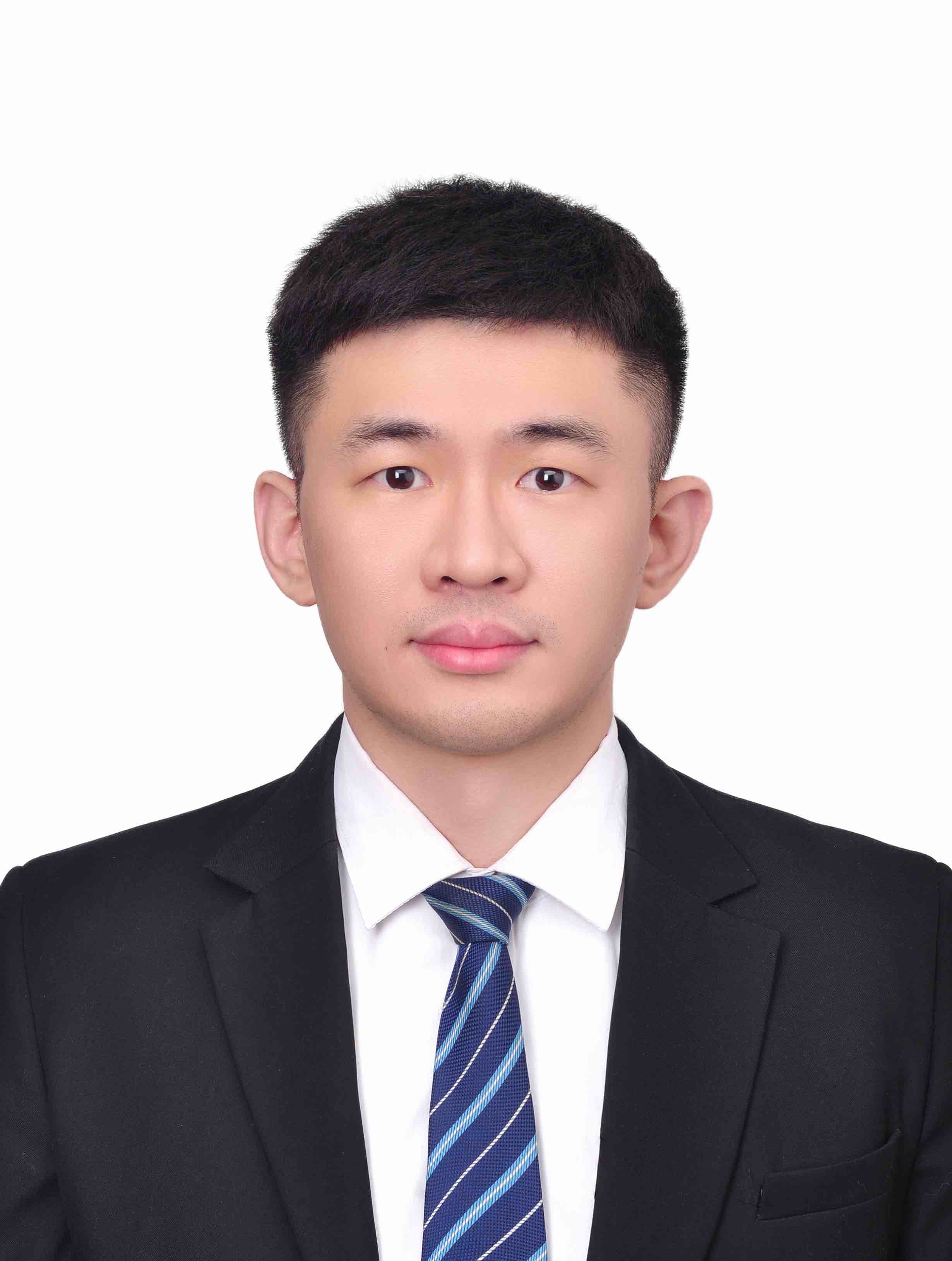}}]{Qiyu Kang}
received the B.S. degree in electronic information science and technology from the University of Science and Technology of China (USTC) in 2015, and the Ph.D. degree from Nanyang Technological University (NTU), Singapore, in 2019. From 2020 to 2024, he was a Research Fellow with NTU. He is currently a Professor with the School of Information Science and Technology, USTC. His research interests include machine learning and computational intelligence.  
\end{IEEEbiography}

\vfill

\end{document}